\pdfoutput=1 %
\documentclass[10pt,a4paper]{article}
\usepackage{amsthm} 
\usepackage{stroop}
\usepackage[accepted]{tmlr}
\PassOptionsToPackage{capitalize}{cleveref}
\usepackage{arydshln}
\usepackage{rotating}  %

\newcommand\defeq{\coloneqq}
\newcommand\minrank[1]{{\operatorname{min~rank}_\Omega(#1)}}
\newcommand\minepsrank[1]{{\operatorname{min~rank}_{\Omega,\varepsilon}(#1)}}

\usepackage{xcolor,mdframed}

\usepackage{xcolor, tcolorbox}

\usepackage{wrapfig}

\title{
On the Low-Rank Parametrization of Reward Models \\ for Controlled Language Generation
}

\author{%
\name Sergey Troshin
\email \eml{s.troshin@uva.nl} \\
\addr Language Technology Lab,
Informatics Institute,
University of Amsterdam
\AND
\name Vlad Niculae
\email \eml{v.niculae@uva.nl}\\
\addr Language Technology Lab,
Informatics Institute,
University of Amsterdam
\AND
\name Antske Fokkens
\email \eml{antske.fokkens@vu.nl}\\
\addr Computational Linguistics and Text Mining Lab, \\ Faculty of Social Sciences and Humanities, Vrije Universiteit Amsterdam
}
\def\month{08}  %
\def\year{2025} %
\def\openreview{\url{https://openreview.net/forum?id=cjRsEGLT8B}} %

\DeclareMathOperator{\rank}{rank}

\newcommand{\myhline}{\noalign{\vskip 0.4ex}\hline\noalign{\vskip 0.8ex}}

\newcommand{\mydashedcline}[1]{\noalign{\vskip 0.4ex}\cdashline{#1}\noalign{\vskip 0.8ex}}

\DeclareMathOperator{\softmax}{Softmax}

\newcommand{\alignedintertext}[1]{%
  \noalign{%
    \vskip\belowdisplayshortskip
    \vtop{\hsize=\linewidth#1\par
    \expandafter}%
    \expandafter\prevdepth\the\prevdepth
  }%
}

\begin{document}
\maketitle
\begin{abstract}

Language models trained on large amounts of data are known to produce inappropriate content in some cases and require careful tuning to be used in the real world. We revisit an effective and modular approach for controllability of the language models, when an external expert model guides the decoding. Particularly, we zoom in into the \textit{parametrization choice} of an external expert, highlighting the difference between low-rank and higher-rank parametrizations. Higher-rank experts are designed to support high flexibility when representing the rewards, leading to higher computational costs during decoding. However, we demonstrate that they might not use their full flexibility. By analyzing the recently proposed reward-augmented decoding approach (RAD), which uses a higher-rank expert model, we introduce a simpler but more efficient low-rank parametrization of the expert model enabling fast and effective guided decoding. We empirically show that the low-rank RAD performs on par with the more flexible RAD on a detoxification and a sentiment control task, while requiring only a single reward model call per generated token. 

\end{abstract}

\section{Introduction}

Generative large language models (LLMs) have gained a lot of popularity in recent years and shown impressive results in zero-shot and few-shot scenarios on numerous downstream tasks \citep{touvron2023llama2openfoundation, openai2024gpt4, jiang2023mistral}. These large-scale models are pretrained on large amounts of data, and are known to inherit and memorize underlying biases \citep{sheng-etal-2019-woman} as well as to provide unsafe responses \citep{wallace-etal-2019-universal, ganguliRedTeamingLanguage2022}, necessitating further tuning for safer deployment and control \citep{ouyangTrainingLanguageModels2022}.

Control over LLMs can be roughly divided into methods which modify the original model via finetuning \citep{ouyangTrainingLanguageModels2022, rafailovDirectPreferenceOptimization2023}, and decoding-time solutions, which do not modify the parameters of the original model, including best-of-n sampling \citep{wang2023selfconsistency, sun2024fast}. As models increase in size, finetuning becomes prohibitive with limited computational resources. In this work, we focus on a more modular approach of decoding-time guidance, and assume we have access to top-$k$ logits of a black-box base language model (see \Cref{sec:preliminaries} for details). In this line of work, a discriminator model is trained to modify or re-rank the logits of the base model during decoding in order to satisfy the desired constraint \citep{yang-klein-2021-fudge}, while preserving the distribution of the language model as much as possible.

In our work, we zoom in on the parametrization choice of the reward model. \citet{yang-klein-2021-fudge},  \citet{deng_2023_rad}, \citet{mudgal2024controlled}, \citet{chakraborty2024transfer} parameterize reward models as discriminators, which requires a separate forward pass through the backbone of the reward model (layers before the output head) for each token candidate. 
\citet{liu-etal-2021-dexperts}, \citet{krause-etal-2021-gedi-generative} and  \citet{caoSystematicRectificationLanguage2022} use more efficient approaches and perform a single forward pass to predict the scores for all next token candidates. A similar parametrization choice is encountered in the reinforcement learning literature (RL), where one chooses between state value function ($V$ function) or state-action value function ($Q$ function) parametrization \citep{sutton2018reinforcement, rl_highdimention}. In particular, 
closest to our work are
the decoding-time RL approaches, where V/Q functions are trained to guide the decoding from a language model \citep{mudgal2024controlled, chakraborty2024transfer}. In our work, we highlight the rank bottleneck  \citep{yang2018breaking} of the $Q$-style parametrized models in application to text reward modeling, where vocabulary size is larger than model dimension. We theoretically analyze the difference between these two parametrizations in terms of rank expressivity.
We provide a simple but realistic example, where $Q$-style parametrization is less expressive compared to $V$-style parametrization, suggesting that the $Q$-style choice of parametrization should not be treated as the default option indiscriminately.

Despite the rank bottleneck limitation, $Q$-style parametrization is more computationally attractive. We therefore ask whether there are cases where it is justifiable to use $Q$-style parametrization, which we answer affirmatively. Particularly, we look at the two popular controlled generation tasks and a recently proposed reward augmented decoding (RAD) approach \citep{deng_2023_rad}, where they train a $V$-style parametrized model (we will refer to it as RAD-V). While RAD-V demonstrates high effectiveness for controlled generation, it scales poorly when the number of next token candidates grows, requiring a separate forward pass through the backbone of the reward model for each token candidate. We demonstrate that for this scenario, high rank is not required to approximate the training data well, and hence RAD-V might not use its full flexibility. To empirically verify that a RAD approach using $Q$-style parametrization, RAD-Q, is then expressive enough for this scenario, we distill RAD-V into an efficient RAD-Q, and demonstrate that guided decoding with the RAD-Q model results in a comparable attribute control/fluency to the more flexible but more computationally intensive RAD-V approach.

\section{Preliminaries}

\begin{figure}[t]
  \centering
  \includegraphics[width=0.7\textwidth]{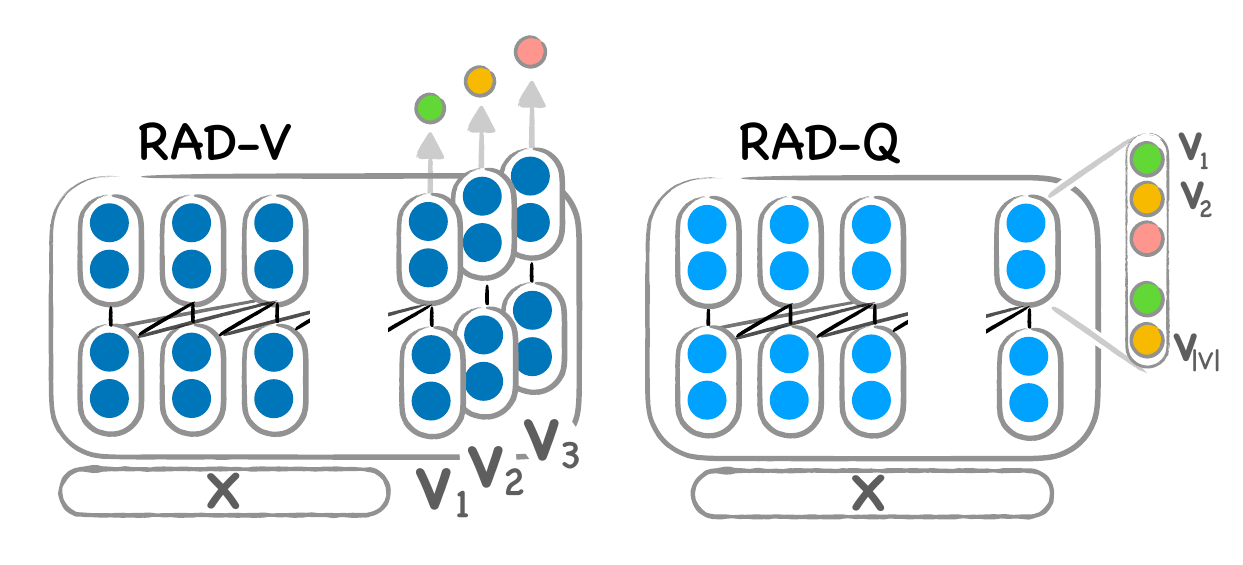}
  \caption{The RAD-V parametrized model (left) predicts a reward for each next token candidate $v$ independently concatenating it to the current prefix $x$. RAD-Q (right) parametrization uses the language model output embeddings to efficiently predict the rewards for next token candidates over the vocabulary $V$.}
  \label{fig:demo}
\end{figure}

\subsection{Guided decoding with external experts}
\label{sec:preliminaries}
In this section, we outline the approach of guiding a base language model with external token-level discriminators. At each step of decoding, both the base model and the discriminator observe an already generated prefix $x$, and cooperate to score the next token candidates $v \in V$. A language model predicts the logits 
 $z_{\text{LM}}(\cdot | x) \in \mathbb{R}^{|V|}$  and the goal of discriminator is to augment these logits with reward scores $\hat{r}(\cdot | x) \in \mathbb{R}^{|V|}$. A standard practice is to consider only likely tokens $V' \subseteq V$ at each decoding step \eg via top-$k$ \citep{fan-etal-2018-hierarchical, deng_2023_rad} or nucleus sampling \citep{Holtzman2020The}:
 \begin{equation}
    z(v|x) = \begin{cases}
      z_{\text{LM}}(v | x) + \beta \hat{r}(v | x), & \text{if} \, v \in V', \\
      -\infty, & \text{otherwise},
    \end{cases}
 \end{equation}
 and the next token is sampled from the categorical distribution:
\begin{equation}
\label{eq:Softmax}
    \tilde{p}(x) = \softmax(z(v|x)).
\end{equation}

While some language models might have a restrictive application programming interface (API) for safety reasons, this line of work makes a reasonable assumption that we have access to the top-$k$ logits of a language model either directly or via API for a relatively small $k \ll |V|$.

To define reward scores, $Q$-style models given a prefix $x$ only pass it once through the external language model backbone, and use the linear output layer to obtain the scores for each of the next token candidates. GeDi \citep{krause-etal-2021-gedi-generative} and DExperts \citep{liu-etal-2021-dexperts} use attribute-conditioned 
 unidirectional language models (undesired attribute in GeDi or two LM experts for desired and undesired attribute in DExperts), trained via the standard language modeling objective on class-conditioned data: $\hat{r}_{y}(v|x) = z_t(v | x, y)$, where $y \in \{0;1\}$ is the attribute (\eg positive/negative sentiment). 

Alternatively, a $V$-style parametrized model such as RAD-V \citep{deng_2023_rad} predicts the attribute of interest for a prefix \emph{concatenated} with a next token candidate $\hat{r}_{\text{RAD-V} }([x, v])$, where $[\cdot,\cdot]$ denotes the concatenation of a prefix and a next token candidate. 
This approach requires passing each next token candidate as \textit{input} to the model, thus, to obtain the scores for $k$ next token candidates $v$ for top-$k$ decoding RAD-V would need $k$ 
  forward calls of the reward model, which can slow down inference significantly and constrains them to limit the number of next token candidates. Therefore, we ask whether $Q$-style parametrized models can possibly match the performance of $V$-style models while enjoying higher efficiency.

\subsection{RAD training}
\label{sec:rad_training}
In this section, we outline how the RAD approach \citep{deng_2023_rad} uses labeled data to train a reward model. 

At the training stage, we assume that we have a dataset $\mathcal{D} = \{(u^{(i)}, y^{(i)})\}_{i=1}^n$ of $n$ text utterances $u$ of length $l(u)$ and responses $y \in [0;1]$.
The RAD approach is to use the data distribution to estimate the expected future response. For each utterance $u$ from the dataset, we can split it prefix $x'$, next token $v$, and a continuation $\mathbf{v}$, in all possible ways. This way, they create an extended dataset to train on partial prefixes:
\begin{equation}
\mathcal{D}_{f} = \{(x, u, y \, | \,x\!=\!u_{1:t}, t\in(1,\ldots,l(u)) ,(u,y) \in \mathcal{D}\}.
\end{equation}
RAD trains a reward model to predict $y$ given a text input. 
Then, during training, RAD takes the input prefix $x=[x', v]$ and incurs a weighted squared loss for approximating the future reward: 
\begin{equation}
\label{eq:cumulative_loss_rad}
    \cL(\hat{r}(v|x'), y, \lambda) = \lambda \cdot (\hat{r}(v|x') - y)^2,
\end{equation} where $\lambda$ are the discounting weights $\lambda(x, u)=l(x)/Z_u$, for each prefix used to up-weight prefixes closer to the full sentence, and $Z_u=\sum_{t=1}^{l(u)} l(u_{:t})$ is a normalizing constant such that $\sum_{t=1}^{l(u)} \lambda(u_{:t}, u) = 1$.
During training, we 
can use teacher forcing to process all prefixes of an utterance in a single pass.

\section{Reward modeling as low-rank matrix factorization}

\subsection{Analysis of RAD} 
\label{sec:matrix_factorization}

\subsubsection{Reward modeling as matrix completion}

To better understand the training objective of RAD, we start by looking at the optimization problems defined in \Cref{eq:cumulative_loss_rad}, where we optimize a reward model to approximate future responses. A unidirectional reward model can predict a reward value for each next next token candidate. If we enumerate all the contexts $x'$ in the training data and all possible next tokens $v$, we task a reward model to predict the values of $R \in \mathbb{R}^{N \times |V|}$, which we dub the \emph{reward matrix}. 

If each context would be observed only once, $R$ would have a single observed reward in each row. For short and common contexts we can observe more continuations per row, and also for some contexts there can be ambiguities: $\{(x, u_1, y_1), ... (x, u_m, y_m)\}$. 
From a mean squared error point of view, it is equivalent to compress these ambiguities by taking the weighted average of their $y$ (\Cref{app:loss_minimizer}):
\begin{equation}
    R[x',v] = \frac{\sum_{u, y \sim \mathcal{D}_f[x]} \lambda(x, u)\, y}{\sum_{u,y \sim \mathcal{D}_f[x]} \lambda(x, u)}.
\end{equation}

From this perspective, reward modeling 
can be interpreted as a matrix completion problem. The training dataset \(\mathcal{D}_f\)
gives us only an incomplete view of a true reward matrix \(R\). 
Following the notation in the matrix completion literature \citep{JMLR:v11:mazumder10a},
$\Omega$ denotes the set of indices of the observed entry indices
$\{(x',v) \,|\, x=[x',v], \,x \in D_f\}$,
and \(P_\Omega(R)\) denotes the projection of \(R\) that sets all indices outside \(\Omega\) to zero. 
We denote the complement of \(\Omega\) \wrt the complete set of indices as \(\neg\Omega\).
The full RAD objective for a $V$-style parametrized model is equivalent to minimizing $\|P_{\Omega}(R)~-~P_{\Omega}(\hat{R}_{\text{RAD-V} })\|_F^2$,
where each entry \(\hat{R}_{\text{RAD-V} }[x',v]=\hat{r}(v|x')\) can be computed with a forward pass. %

\begin{figure}[ht!]
  \centering
  \includegraphics[width=0.90\textwidth]{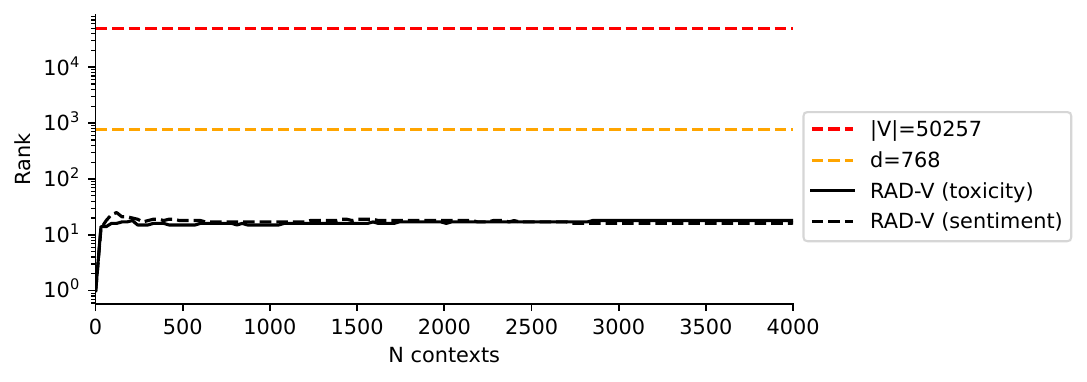}
  \caption{When considering a RAD-Q approach \eg distilling RAD-V into RAD-Q, it is important to ask whether the rank capacity of the RAD-Q model is enough to model RAD-V outputs. We numerically estimate the rank of $\hat{R}_{\text{RAD-V} }$ by incrementally considering more rows of the $\hat{R}_{\text{RAD-V}} $ matrix, and observe that the rank tends to be less than the model dimension $d=764$ and much less than $|V|$, the maximal possible rank of $P_{\Omega}(R)$.}
  \label{fig:rank_of_rad}
\end{figure}

 To analyze the reward matrices from the perspective of rank expressivity, we define the \textit{minimal rank} 
of a partially observed matrix
\begin{definition}
    Let $R$ be a $n \times m$ matrix, with $\Omega$ denoting a set of observed indices. 
    We define the minimal rank of $R$ \wrt $\Omega$ as
    the rank of the lowest-rank completion consistent with \(P_\Omega(R)\):
    \begin{equation}
        \minrank{R} \defeq \min \left\{\rank(\hat{R}) :
    \hat{R} \in \bbR^{n \times m}, P_\Omega(R) = P_\Omega(\hat{R}) \right\}
    \end{equation}
\end{definition}
Additionally, since numerical solutions to low-rank problems must incur some error, we also define a \emph{minimal numerical rank}.
\newpage
\begin{definition}
\label{definition:minimal_numerical_rank}
    Let $R$ be a $n \times m$ matrix, with $\Omega$ denoting a set of observed indices, and \(\varepsilon \in \bbR_+\).
    We define the minimal $\varepsilon$-rank of $R$ \wrt $\Omega$ as:
    \begin{equation}
        \minepsrank{R} \defeq 
        \min \left\{\rank(\hat{R}) :
    \hat{R} \in \bbR^{n \times m},
    \frac{1}{nm}{\| P_\Omega(R) - P_\Omega(\hat{R})\|}^2_F \leq \varepsilon \right\}.
    \end{equation}
\end{definition}

\subsubsection{RAD-V can be high-rank, but is not in practice}

Given a prefix $x$, RAD-V accepts a token candidate $v$ as an additional \textbf{input} to the model  $\hat{R}_{\text{RAD-V} }[x',v]=\hat{r}_{\text{RAD-V} }([x',v])$, passing $v$ through the layers of the reward model. For this reason, we expect RAD-V to have the capacity to represent a large space of reward matrices including matrices with higher rank. In \Cref{app:rad_rank}, we empirically verify that \textbf{RAD-V is capable to 
model data that has high minimal rank},
\ie, it can fit (with training loss numerically close to zero) a dataset $R$ that has $\minrank{R}>d$, where $d$ is the dimensionality of the Transformer model.
This flexibility does come at a cost: to score many next token candidates during top-$k$ decoding, RAD-V requires a forward pass through all layers of the model for each of the $k$ next token candidates. Following this observation, an important question is \emph{do we need this flexibility at the cost of slower decoding?}

In \Cref{fig:rank_of_rad}, we aim to measure the rank of $\hat{R}_{\text{RAD-V} }$ for RAD-V trained on two datasets: for detoxification and sentiment control tasks (discussed in detail in \Cref{sec:experiments}). To numerically estimate the rank, we follow \citet{finlayson2024closing} and first sample $N$ random prefixes $x$ from the dataset $\mathcal{D}_f$ to calculate $N$ full rows of $\hat{R}_{\text{RAD-V} }$ (requiring $N \cdot |V|$ calls to the RAD-V reward model). Then we use singular value decomposition with the standard singular value cutoff to compute the \textit{numerical rank} (\Cref{app:matrix_rank}). \textbf{We observe that the reward matrix learned by RAD-V tends to have \emph{low numerical rank}}, suggesting that it is possible to use less flexible but faster reward models to improve the efficiency of reward models.

\subsubsection{RAD training data is low-rank.}

One possible explanation why RAD-V does not use its full flexibility is that said flexibility is not required to fit the training data, \ie, that $P_{\Omega}(R)$ can be fit with the low-rank model. To analyze the rank needed to fit the data, we use the definition of the minimal numerical rank (\Cref{definition:minimal_numerical_rank}). Empirically calculating the minimal numerical rank of the data is challenging due to the very large number of prefixes. We use a combination of theoretical and empirical approaches listed in \Cref{sec:is_data_low_rank}, and we claim that the RAD training data (Jigsaw \citep{jigsaw-unintended-bias-in-toxicity-classification} and Amazon Polarity \citep{zhang2015character})  \textbf{has low minimal numerical rank}, that is less than the model dimension for $\varepsilon=10^{-6}$, and even less than $256$ if we allow $\varepsilon=10^{-3}$. We further confirm this claim for the commonly used  HelpSteer (helpfulness) dataset \citep{wang-etal-2024-helpsteer} and BeaverTails (safety) dataset \citep{NEURIPS2023_beavertails}, used for reward model training \citep{wang2024interpretablepreferencesmultiobjectivereward}. We thus draw the conclusion that the incomplete $P_{\Omega}(R)$ matrix can be fit with the low-rank matrix factorization with a small error. 

\subsubsection{Is low-rank enough to model reward matrices?}
When viewed from a matrix completion angle, it may seem that the reward modeling problem is well suited for low-rank modeling. However, \textbf{we will argue this should not be taken for granted}.

First, it is indeed true that data matrices generated by a certain kind of random process \citep{udell_random_rank} are low-rank with high probability. Moreover, in the case of missing values, intuition suggests that they should tend to low-rank, and indeed the following result demonstrates that we can almost always decrease the rank by $1$ when filling in one missing value (\Cref{app:lemma_1}):
\begin{lemma}
For $k>0$, let $R$ be a
random matrix with distribution supported on $\bbR^{k \times k}$, 
and a single missing value, \ie, $\neg\Omega = \{(i,j)\}$. 
Then, 
$P(\minrank{R} < k)=1$.
\end{lemma}
Moreover, when most values are missing, rank again seems likely to be low.  In particular, when only one value is observed per row (as demonstrated in \Cref{example:1}), a rank-1 completion is possible (\Cref{example:2}, \Cref{app:rank_1}). In a dataset, there will likely be many unique prefixes where this scenario is applicable, particularly for longer prefixes.

\begin{example}
\label{example:1}
As illustrated in \Cref{fig:sparse}, assume that all prefixes $x$ appear only once in the dataset, so that $\left|\{(i,j) \in \Omega : i=x\}\right|=1$ for all $x$. 
Then, \(\minrank{R}=1\).
\end{example}

\begin{example}
\label{example:2}
As illustrated in \Cref{fig:context_independent}, let each token in the vocabulary either appear in high-reward ($y$=1) or low-reward ($y$=0) utterances (independent of context). 
This gives a rank-1 completion consistent with \Cref{example:1}.
\end{example}

Given this observation, one might assume that having more missing values in a reward matrix always implies that the reward matrix can be fit with lower rank. However, this is not true, and the following result suggests that a reward matrix with many missing values and can still be forced to have a high rank:
\begin{lemma}
\label{lemma:high_rank}
For $k>0$, there exist $R \in \bbR^{k \times k}$
and $\Omega$ with $|\Omega| \in O(k^2)$
such that $\minrank{R}=k$.
\end{lemma}
\begin{proof}
    Consider a $k$-by-$k$ identity matrix with $k(k-1)/2$ missing values above the diagonal (see \Cref{fig:full_rank} from \Cref{example:3}).
    Any possible completion results in an upper-triangular matrix.
    The determinant of such a matrix is the product of its diagonal elements, \ie, 1, so any completion must be full rank.
\end{proof}
The construction in \Cref{lemma:high_rank} is not artificial and could indeed correspond to a reward modeling task for the context-dependent constraints, as shown in the next examples. 
\begin{example}
\label{example:3} Consider a list of (first name, last name) pairs: [($k_i, v_j$)], $1\le i,j \le n$, where some pairs are marked as allowed or blocked. Assume that the rule is as follows: for any utterance containing $(k_i, v_j)$ with $j \le i$, we assign the reward $0$; if $j = i$, we assign the reward $1$, and unknown otherwise (\Cref{fig:full_rank}). We assume that each utterance contains a single mention of a person, then the expected reward matrix with $[\ldots, k_i]$ indexing rows and $v_j$ indexing columns will contain the upper-triangular submatrix, which is of rank $n$ .
\end{example}

\begin{example}\label{example:3b}
Consider an application of a LM to an arithmetical reasoning task, where you have examples in form of $x + y = z$, where $x$, $y$, $z$ are integer number tokens forming a vocabulary $V=[0..999]$. Let us define a reward function as $R(x + y = z) = 1$ if the expression is true, $0$ if false. Then by considering all possible contexts $x + y =$, it is easy to show that the reward matrix is full-rank since the rows of the reward matrix include all one-hot vectors representing the correct answer.
\end{example}
Empirically, our results find that Jigsaw \citep{jigsaw-unintended-bias-in-toxicity-classification} and Amazon Polarity \citep{zhang2015character} indeed support low-rank completions (\Cref{sec:experiments}). 
Nevertheless, from all results in this section, we conclude
that care must be taken before making low-rank assumptions about
reward datasets.

\begin{figure}[t]
    \centering
    \hspace{6mm}
    \begin{subfigure}{0.26\textwidth}
        \centering
        \includegraphics[width=\textwidth]{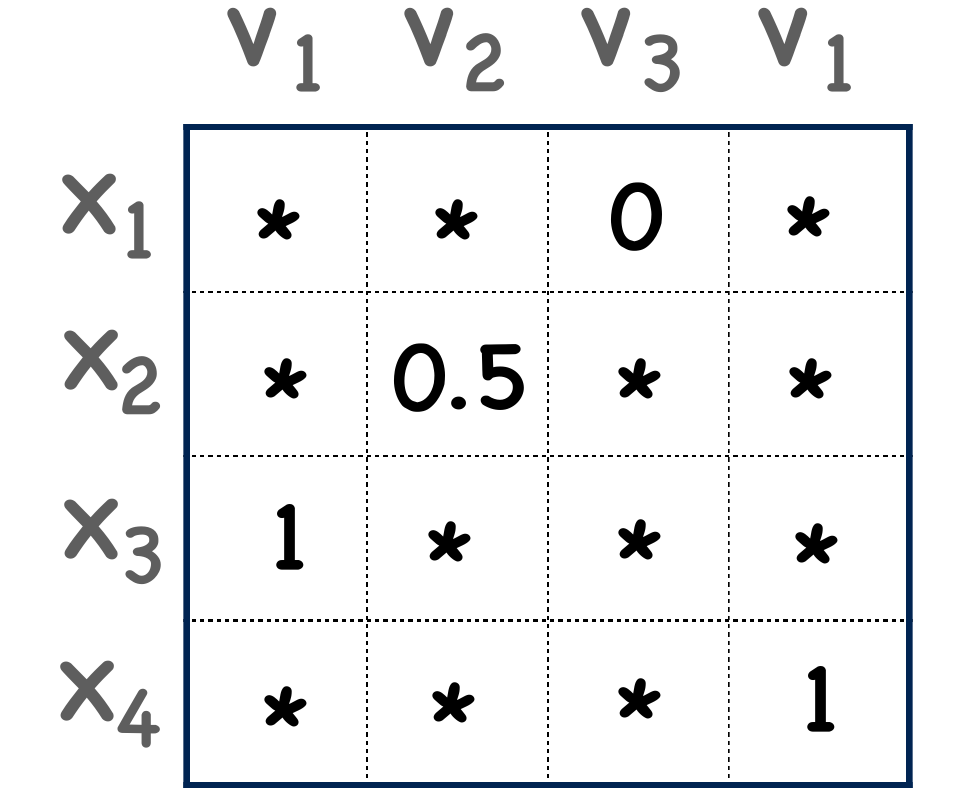}
        \caption{Sparse $P_{\Omega}(R)$}
        \label{fig:sparse}
    \end{subfigure}
    \hspace{10mm}
    \begin{subfigure}{0.26\textwidth}
        \centering
        \includegraphics[width=\textwidth]{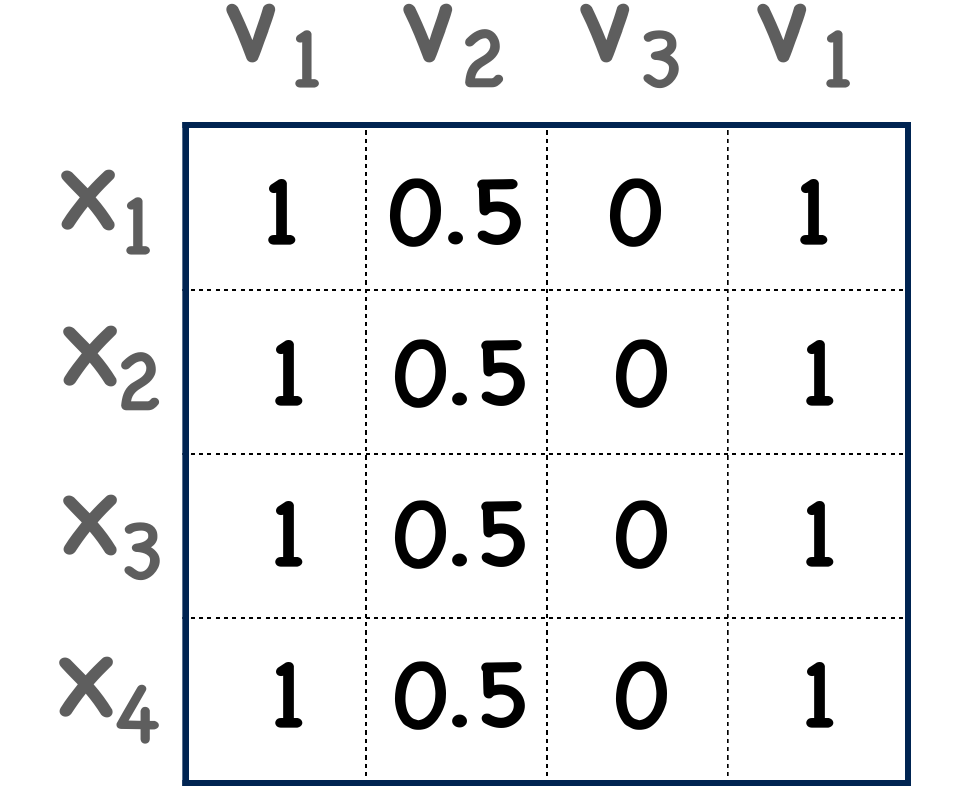}
        \caption{Context independent $R$.}
        \label{fig:context_independent}
    \end{subfigure}
    \hspace{10mm}
    \begin{subfigure}{0.26\textwidth}
        \centering
        \includegraphics[width=0.9\textwidth]{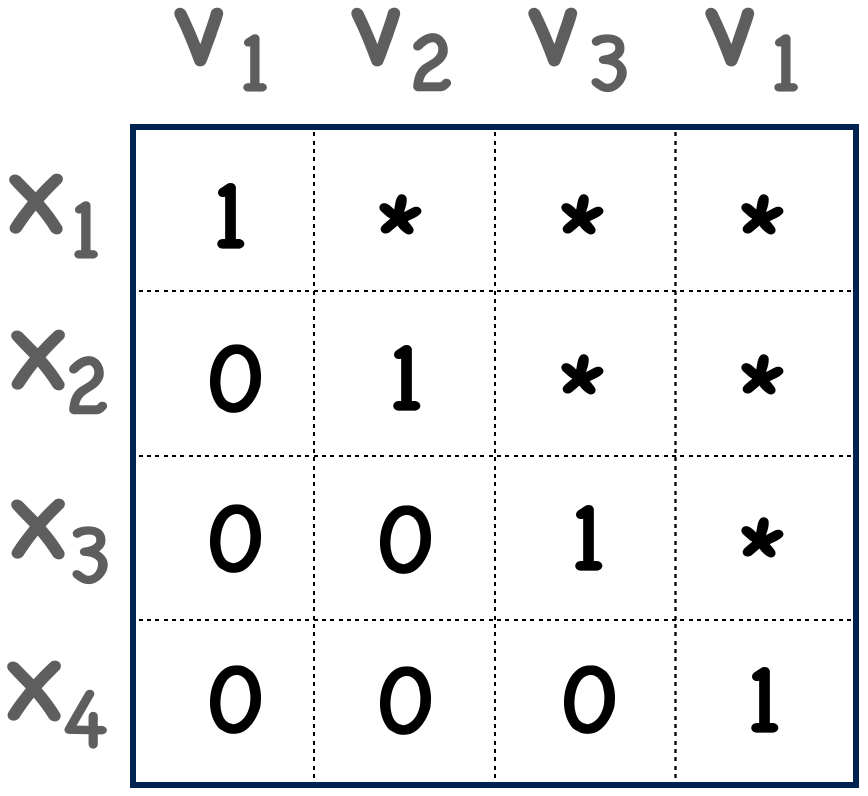}
        \caption{Full rank $P_{\Omega}(R)$.}
        \label{fig:full_rank}
    \end{subfigure}
    \hspace{6mm}

    \caption{(a) each row of a reward matrix $P_{\Omega}(R)$ has only $1$ known element, the minimal rank is $1$; (b) a reward matrix has rank $1$, when each token has a context independent reward; (c) a reward matrix is full rank for highly context dependent rewards e.g.\ for the case of identity matrix.}
    \label{fig:three_plots}
\end{figure}

\subsection{Low-Rank Autoregressive Reward Model}
In this section, we introduce RAD-Q (\Cref{fig:demo}), a low-rank parametrization of RAD, designed for efficient modeling of reward scores for next token candidates.

We revisit the language modeling ($Q$-style) parametrization for reward prediction \citep{liu-etal-2021-dexperts, krause-etal-2021-gedi-generative} and predict the scores for all next token candidates with a single forward pass through the backbone of a language model. In contrast to RAD-V, RAD-Q predicts the representation vector $h(x) \in \bbR^d$ given a prefix $x$ and uses output embeddings $e(v) \in \bbR^{d}$ to get the scores for all next token candidates.
We use the following RAD-Q parametrization, similar to the how Dueling Network \citep{pmlr-v48-wangf16, pmlr-v202-tang23h, han2024valueaugmentedsamplinglanguage} parametrizes the scores for the next tokens given the prefix:
\begin{equation}
\label{eq:softplus}
    \hat{r}_{\text{RAD-Q}}(v| x) = \underbrace{\hat{r}_b(h(x))}_{\text{baseline}}
    + \Delta \hat{r}(e(v) | h(x)),
\end{equation}
where the baseline predicts the score for the prefix $x$ and $\Delta \hat{r}$ predicts how observing a next token $v$ changes the score. Particularly, we use a  \emph{linear parametrization}: %
\begin{equation}
\label{eq:linear_arm}
    \hat{r}_{b}(v| x) := \langle h(x), w \rangle, \qquad  \Delta \hat{r}(e(v) | h(x)) := \langle h(x), We(v) \rangle.
\end{equation}
Here, we introduced two attribute-specific parameters: $w\in \bbR^{d}$ for modeling the baseline reward score of the prefix, and $W\in \bbR^{d \times d}$ to model marginal rewards for each next token candidate. 
\paragraph{Rank bottleneck of RAD-Q.}
Now it is clear that in contrast to RAD-V, RAD-Q (as defined in \Cref{eq:linear_arm}) performs a \emph{low-rank matrix factorization} of $P_{\Omega}(R)$:
\begin{equation}
\label{eq:rad_factors}
    \hat{R}_{\text{RAD-Q}} = H(w \mathbf{1}^T + WE) = HA,
\end{equation}
where we stack all context representations $x'$ into $H \in \mathbb{R}^{N \times d}$ and all next token representations into $WE \in \mathbb{R}^{d \times |V|}$, and $\mathbf{1}$ is a column $d$-vector of all ones. By the rank inequality, 
\begin{equation}
\rank(\hat{R}_{\text{RAD-Q}}) = \rank(HA) \le \min(\rank(H), \rank(A)) \le d,
\end{equation}
meaning that if $\minrank{R}>d$, RAD-Q \textbf{cannot} possibly perfectly reconstruct $P_{\Omega}(R)$ no matter how flexible $h(x)$ is. In terms of $\varepsilon$-rank, if $\minepsrank{R} > d$, then RAD-$Q$ cannot fit the data with training loss less than $\varepsilon$. In the language modeling literature, the rank bottleneck problem is known as the \emph{softmax bottleneck} \citep{yang2018breaking} 
and mitigation strategies are well-studied \citep{ganea2019breaking,chang-mccallum-2022-softmax}.

In the experiments (\Cref{sec:experiments}), we empirically demonstrate that 
distillation of RAD-V into low-rank RAD-Q
can match the performance of the more flexible RAD-V on the two standard controlled generation benchmarks. 

\subsection{RAD-Q training}

To train RAD-Q, we rely on the RAD approach to train a reward model. We split $x$ into a last token and remaining prefix: $x=[x',v]$. We pass $x'$ as input to the model, and $v$ indexes output embeddings~\eqref{eq:linear_arm}.
We consider two types of experiments: training RAD-Q on original responses from the dataset, and a distillation experiment, where we train RAD-Q to predict the scores of less efficient RAD.

For the first type of experiment, we train RAD-Q on the responses from the dataset using the weighted squared loss:
\begin{equation}
\label{eq:cumulative_loss}
    \cL(\hat{r}(v|x'), y, \lambda) = \lambda (\hat{r}(v|x') - y)^2
\end{equation}
For the second type of experiment, we train an RAD-Q student to approximate the less efficient RAD-V teacher $\tilde{r}(x)$ (a frozen trained RAD) using the \emph{distillation loss} \citep{hintonDistillingKnowledgeNeural2015}:
\begin{equation}
\label{eq:distil_loss}
    \cL_{\text{dstl}}(\hat{r}(v|x'), \tilde{r}(x)) = (\hat{r}(v|x') - \tilde{r}(x))^2.
\end{equation} 
A reward model can only observe a limited number of next tokens $v$ given $x$ during finetuning. While the loss defined above provides a positive signal for some tokens $v$, it might be beneficial to regularize the prediction for other (unrelated) tokens, including rare or unseen tokens. In our parametrization~\eqref{eq:softplus}, it is natural to push the predicted reward towards the baseline for unrelated tokens. 
We regularize the prediction of RAD-Q to be close on average to the prefix baseline by forcing $\Delta \hat{r}$ to be close to $0$ for randomly sampled token candidates:
\begin{equation}
\label{eq:reg_loss}
    \cL_{\text{reg}}(h(x)) = \mathbb{E}_{v' \sim \text{Uniform}[V]}\left[\Delta \hat{r}(e(v') | h(x))\right]^2,
\end{equation}
where we use one sample of $v'$ for each prefix position, sampling uniformly from the vocabulary.
Particularly, a regularized model can learn to \textit{abstain} by predicting the baseline score for each next token candidate, which will not change the distribution of a base model.

\section{Experiments}
\label{sec:experiments}

\subsection{Controlled generation}

We follow previous work \citep{deng_2023_rad, liu-etal-2021-dexperts} and evaluate RAD-Q on two controlled generation tasks: detoxification and sentiment control \footnote{Our code is available at  \url{https://github.com/serjtroshin/rad-q}
}.

In our experiments, we guide the decoding from a base model using a smaller finetuned reward model with the same tokenizer. Namely, we guide GPT-2-Large using a reward model finetuned from GPT-2-Small, and we guide the LLaMa-2-(7b/13b) \citep{touvron2023llama2openfoundation} base language model with a reward model finetuned from TinyLLaMa \citep{zhang2024tinyllama}. We finetune all parameters of the reward models except input/output embeddings, which remain frozen 
(in hope of improving generalization to unseen tokens).

We conduct experiments in two regimes: first, by distilling less efficient RAD-V \citep{deng_2023_rad} using $\cL_{\text{dstl}}$ 
 loss~\eqref{eq:distil_loss}; second, by training a reward model from scratch on the responses from the datasets using cumulative loss $\cL$~\eqref{eq:cumulative_loss}. 
 In both settings, we use additional regularization $\cL_{\text{reg}}$ by default. For evaluation, we perform guided decoding using \emph{top-k} sampling from the categorical distribution defined in~\eqref{eq:Softmax}, where \emph{top-k} candidates are selected taking \emph{k} largest logits of the base model at the current decoding step.

\subsection{Detoxification}

\begin{figure}[t]
  \centering
    \centering
    \includegraphics[width=\textwidth]{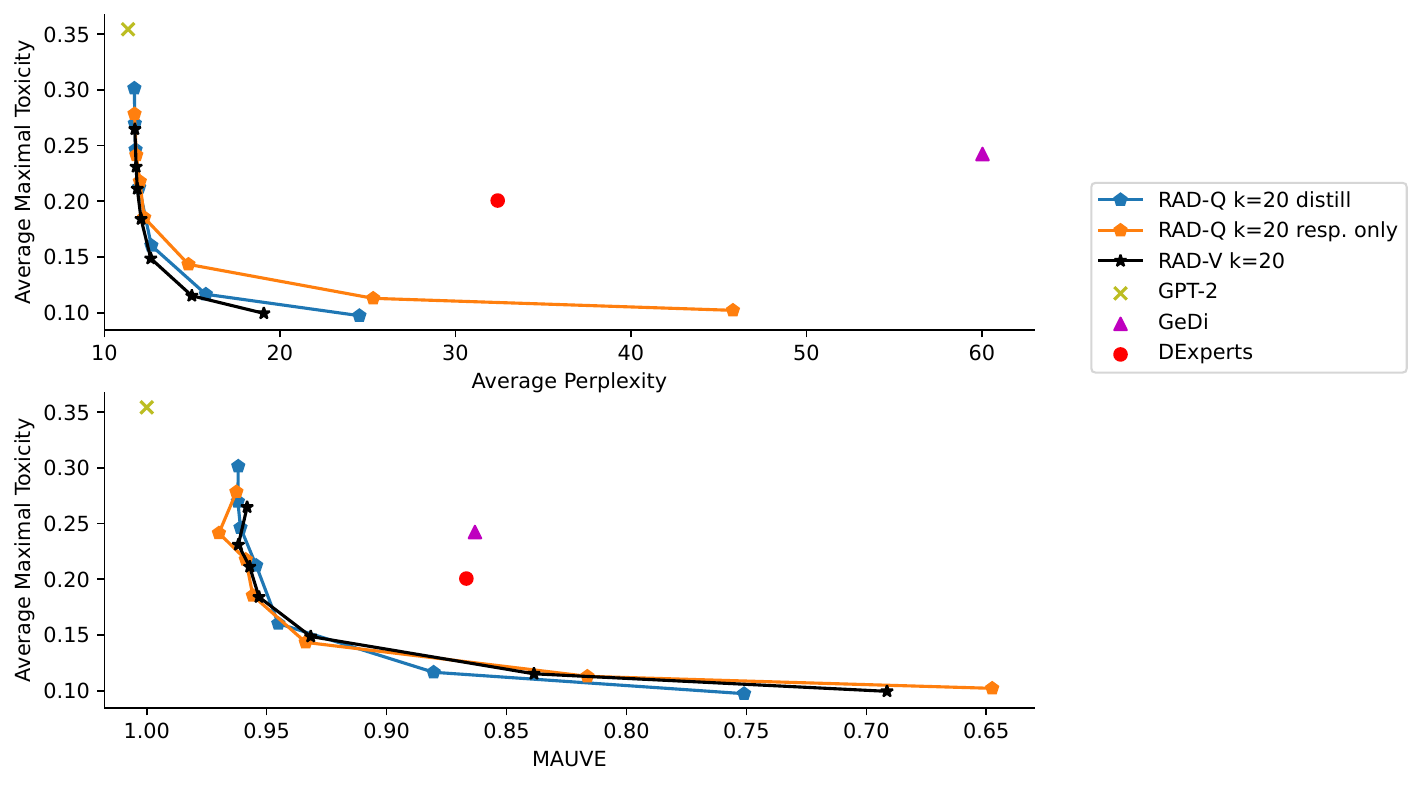}
    \caption{RAD-Q student (distill) shows comparable toxicity/fluency trade-off with the teacher RAD-V, where the RAD-Q student closely matches the performance of the teacher RAD. RAD-Q trained on original responses (RAD-Q resp. only) shows slightly worse fluency and similar toxicity level. We rerun the evaluation for RAD-V, GeDi and DExperts with an up-to-date Perspective API classifier. We include the results with other baselines from \cite{deng_2023_rad} in \Cref{fig:deng_et_al_toxicity_results} (see \Cref{sec:app_detox}). }
    \label{fig:toxicity}
  
\end{figure}

For the detoxification evaluation, we follow previous work \citep{deng_2023_rad, liu-etal-2021-dexperts} and evaluate samples from guided decoding given a \text{10k} subset \citep{liu-etal-2021-dexperts} of prompts from the RealToxicityPrompts dataset \citep{gehmanRealToxicityPromptsEvaluatingNeural2020}. We follow \citet{deng_2023_rad} and \citet{liu-etal-2021-dexperts} and finetune our model on \text{2M} pairs of text and continuous `toxicity' responses between $0$ and $1$ from the Jigsaw Unintended Bias in Toxicity Classification challenge \citep{jigsaw-unintended-bias-in-toxicity-classification}. Like previous work, we train our model on $7$ independent responses (`toxicity', `severe toxicity', `obscene',
            `identity attack', `insult', `threat', `sexual explicit') with different head parameters $w_i, W_i, i \in \{1, ..., 7\}$ for each sub-task. During decoding, we only use the `toxicity' predictor.
For the distillation experiment, we use the same dataset, and the released toxicity discriminator from \citet{deng_2023_rad} as a teacher.

During decoding, we sample $25$ continuations generating at most $20$ new tokens. To evaluate toxicity, we use an external closed-source toxicity classifier \emph{Perspective API} \citep{leesNewGenerationPerspective2022}, and following previous work \citep{deng_2023_rad, liu-etal-2021-dexperts}, we rely on the \emph{Maximal Average Toxicity} metric, which is the maximal toxicity score value over $25$ samples for a given prompt, averaged over the set of \text{10k} prompts. We also report \emph{Toxic Rate}, which is calculated as the probability that at least one out of 25 continuations is
toxic according to Perspective API (toxicity score > $0.5$); and \emph{Diversity} score, which is the average number of distinct $n$-grams normalized by the length of text \citep{liDeleteRetrieveGenerate2018}.
To evaluate the fluency of model generations, we follow previous work \citep{liu-etal-2021-dexperts, deng_2023_rad} and report the average perplexity of the GPT-2-XL when generating from the GPT-2-Large model; and the OLMo\footnote{\url{https://huggingface.co/allenai/OLMo-1B}} to evaluate the LLaMa family as in \cite{lovelace-etal-2024-diffusion}. As an additional fluency metric, we report MAUVE \citep{pillutla2021mauve} to measure the distance between unguided and guided generations (details in \Cref{app:mauve}). In the experiments, we will look at the toxicity/fluency trade-off, alternating the weight $\beta$ of the discriminator (see \Cref{tab:full_toxicity} and \Cref{tab:full_sentiment}). We expect to obtain a model with both low toxicity according to the Perspective API, and high fluency. 

Since toxicity scores from the Perspective API can change overtime, which can complicate the evaluation, in  \Cref{app:roberta} we evaluate our detoxification models with an open-weight toxicity classifier \footnote{\url{https://huggingface.co/nicholasKluge/ToxicityModel}}, where we observe the same relative results as with Perspective API scores.

\begin{figure}
    \centering
    \includegraphics[width=\textwidth]{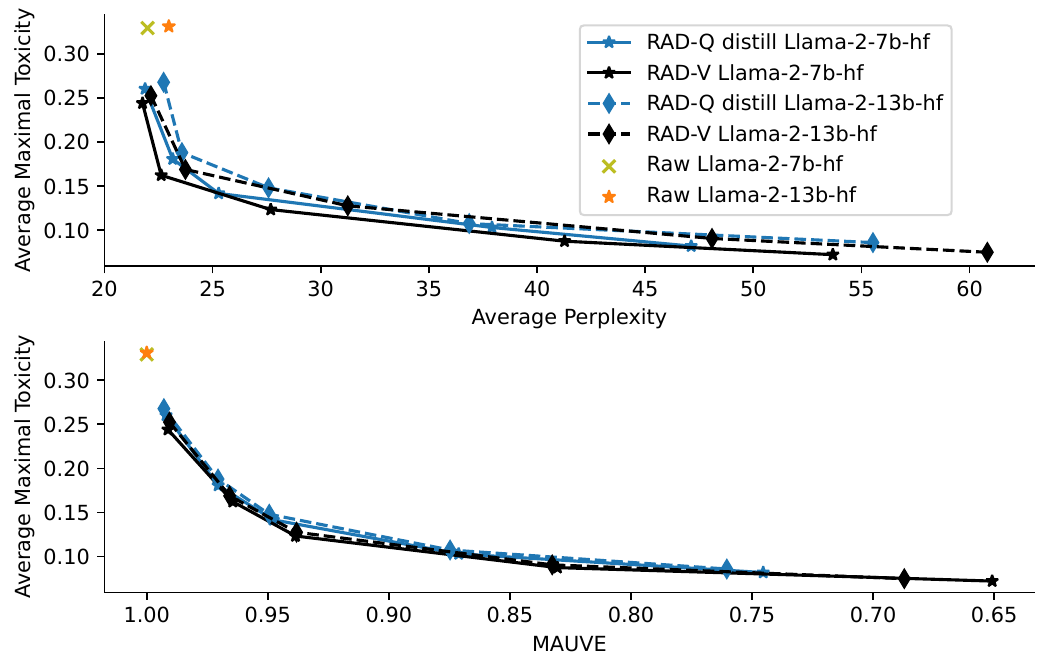}
    \caption{Detoxification results with Perspective API toxicity classifier and LLaMa-2 model. RAD-V and RAD-Q (distill) demonstrate similar performance \wrt two fluency metrics: average perplexity and MAUVE. Performance remains consistent across different models sizes of base LLaMa-2 model.}
    \label{fig:llama_toxicity_perspective}
\end{figure}

\subsection{Sentiment control}
\begin{figure}[t]
  \centering
  \begin{minipage}{0.9\textwidth}
  \includegraphics[width=\textwidth]{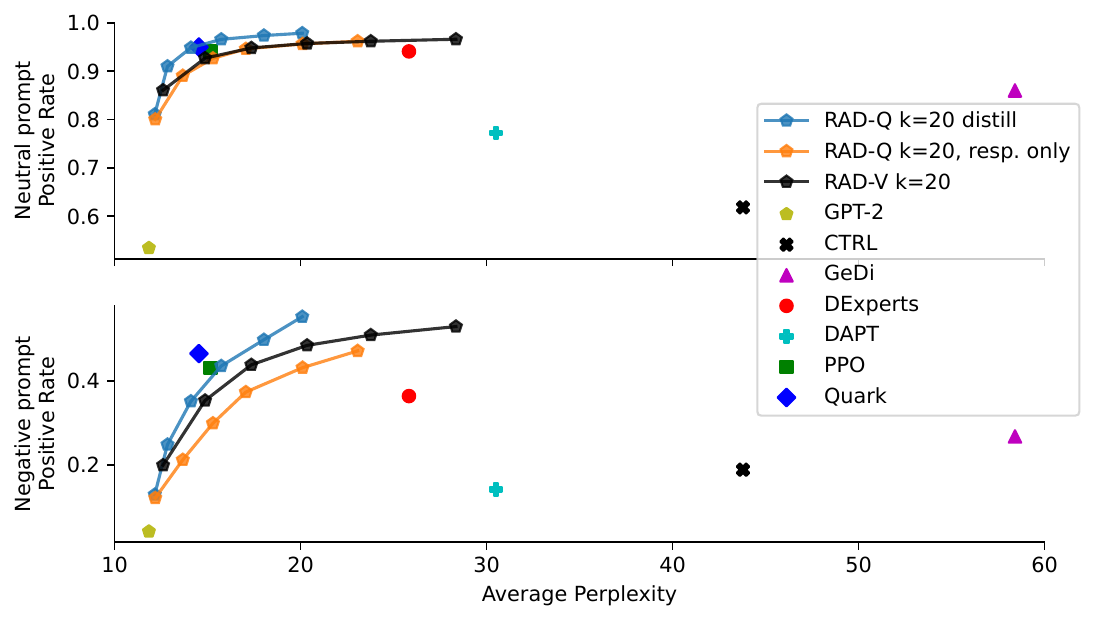}
  \label{fig:sentiment}
  \end{minipage}

  \centering
  \begin{minipage}{0.9\textwidth}
  \includegraphics[width=\textwidth]{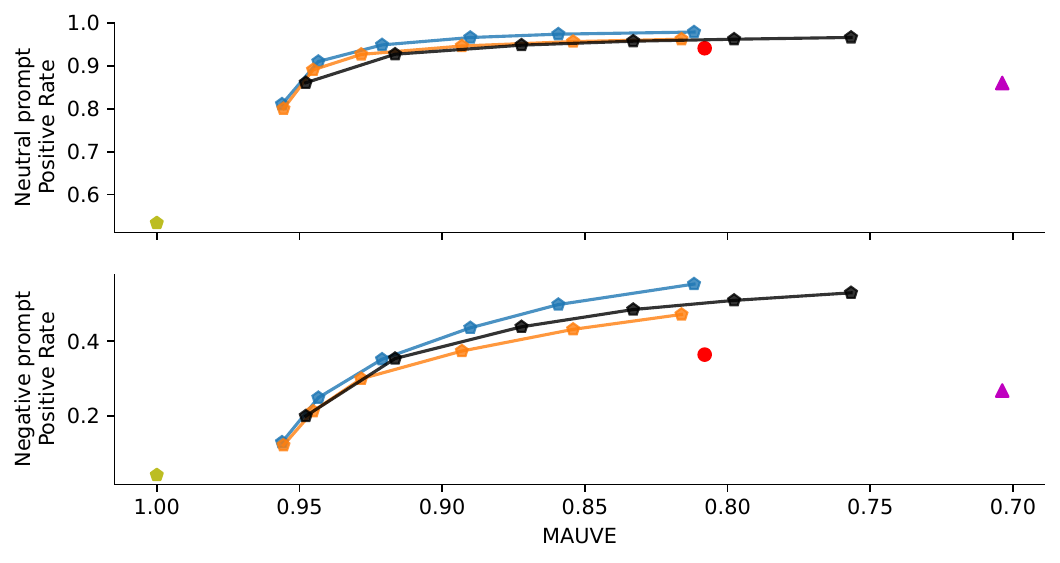}
  \caption{For the sentiment control task, RAD-Q trained on responses only lags slightly behind the RAD-V baseline, while student RAD-Q outperforms the teacher RAD-V model. For the plot with average perplexity, we include the results from \cite{deng_2023_rad} for other baselines for reference.}
  \label{fig:sentiment_main_text}
  \end{minipage}
\end{figure}

For sentiment control, we follow previous work \citep{liDeleteRetrieveGenerate2018, sudhakarTransformingDeleteRetrieve2019, liu-etal-2021-dexperts, deng_2023_rad} to evaluate the samples given a prompt from one of the three categories: $2.5K$ \emph{negative}, $5K$ \emph{neutral}, and $2.5K$ \emph{positive} prompts from OpenWebText \citep{gokaslan2019openwebtext}. To finetune RAD-Q on responses only, we follow \citet{deng_2023_rad} and finetune our model on millions of reviews from the Amazon Polarity \citep{zhang2015character} and 
 SST-2 \citep{socher2013recursive} datasets. 
To distill the sentiment discriminator of \citet{deng_2023_rad}, we use text examples from the Amazon Polarity dataset. Additional training details are provided in \Cref{app:training_details}.

For evaluation, we follow \citet{deng_2023_rad}, and use the average \emph{Positive Rate} metric \wrt the finetuned DistilBERT classifier \citep{sanh2019distilbert} provided via HuggingFace\footnote{\url{https://shorturl.at/9MqDp}}. As in the toxicity task, we use GPT-2-XL/OLMo and MAUVE to evaluate the fluency of the sampled continuations, and we expect to obtain a high Positive Rate and high fluency.

\subsection{Results}
To compare RAD-V and RAD-Q, we rely on the methodology of \citet{deng_2023_rad} and \citet{liu-etal-2021-dexperts}, and visualize the trade-off plots for both models varying the control parameter $\beta$. Namely, each point in the figure will represent two metrics: toxicity/sentiment along the vertical axis and fluency along the horizontal axis. From this plot, we can read, \eg, what fluency (perplexity/MAUVE) can be achieved for a given `target' toxicity. To compare two models, we compare their curves (in the same plot). Our hypothesis is that RAD-Q will perform similar to the more flexible RAD-V approach, meaning that the trade-off plots for these models will be close to each other.

\paragraph{Detoxification.}
For the detoxification task (\Cref{fig:toxicity} with GPT-2, \Cref{fig:llama_toxicity_perspective} with LLaMa-2), our efficient student (RAD-Q) closely follows the RAD-V teacher for toxicity control/fluency trade-off. 
We observe that RAD-Q trained on  responses only shows slightly worse fluency \wrt average perplexity for lower levels of toxicity. 
For completeness, in \Cref{fig:deng_et_al_toxicity_results}, we include the results for other baselines from \citet{deng_2023_rad} computed for an older version of Perspective API.
For guided decoding from the LLaMa-2-(7b/13b), we observe that again RAD-Q closely follows RAD-V in terms of toxicity/fluency trade-off (see \Cref{fig:llama_toxicity_perspective} in \Cref{app:detox_perspective}).

\paragraph{Sentiment control.}
From the results on the sentiment control task in \Cref{fig:sentiment_main_text}, we observe that the RAD-Q student model shows slightly better trade-off than the RAD-V teacher model, closely following approaches that require training using feedback from the evaluation pipeline \citep[Quark]{luQUARKControllableText2022}, \citep[PPO]{ stiennonLearningSummarizeHuman2020}. Again, RAD-Q trained on original responses slightly lags behind but still performs competitively compared to other guided decoding baselines.

\paragraph{Summary.}
 First, we observe that distilled RAD-Q can match or even outperform RAD-V, which confirms that $Q$-style parametrization is expressive enough for our controlled generation scenarios. Second, distilling the RAD-V teacher into the RAD-Q student results in slightly higher quality compared to training RAD-Q on original responses. One difference is that when training from data, we will see short contexts multiple times with different reward responses and must implicitly converge to their average, while in distillation, the teacher already performs this compression and provides a single deterministic target $\hat{r}(v|x)$ for every context $(x,v)$. We conjecture that this may lead to better-trained distilled models.

\begin{figure}[tbp]
  \centering
  \subfloat[Rank analysis]{
    \includegraphics[width=0.47\textwidth]{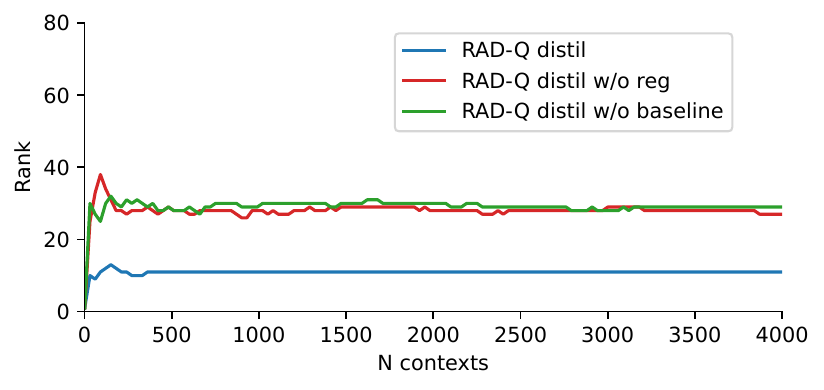}
 }
 \hfill
 \subfloat[Detoxification task]{
    \includegraphics[width=0.47\textwidth]{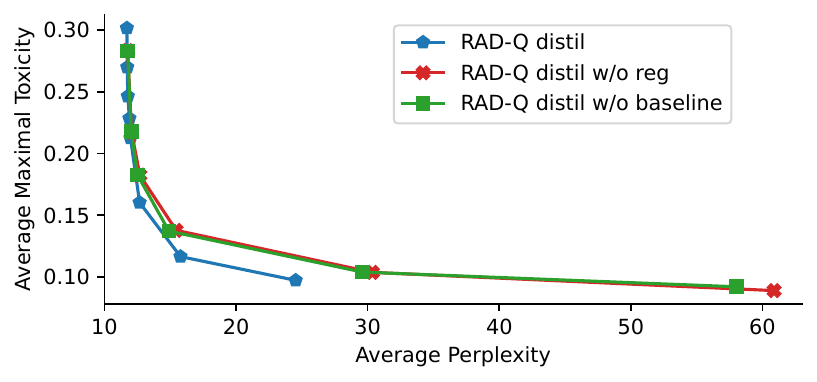}
  }
 9
  \caption{Ablation experiment for distilled RAD-Q, on the detoxification task with top-$k$=$20$. On the right, we observe that regularization towards the baseline results in better fluency of generated samples. On the left, we observe that regularization lowers the rank of the model's outputs $\rank (\hat{R}_{\text{RAD-Q}})$.
  }
  \label{fig:ablation}
\end{figure}

\subsection{Ablation}
In this section, we investigate the effect of adding the baseline component \Cref{eq:linear_arm} and of regularization \Cref{eq:reg_loss}. In \Cref{fig:ablation}, we experiment with the distilled version of RAD-Q and observe that turning off regularization, or further removing the baseline from the parametrization results in still adequate but slightly worse fluency as measured by perplexity, and a comparable toxicity decrease. By further analyzing the ranks of $R_{\text{RAD-Q}}$ with and without regularization, we observe that regularization effectively decreases the rank of $R_{\text{RAD-Q}}$, which might explain the higher fluency of regularized models. Particularly, strong regularization would result in the model always predicting the baseline score for each of the next tokens (corresponding to the rank-1 output), which does not modify the original distribution of the model (the best fluency).

\begin{figure}[t]
  \centering
  \begin{minipage}{0.45\textwidth}
    \centering
    \begin{tabular}{ll}
        \toprule
        \textbf{Model} & \textbf{N calls} \\
        \midrule
        GeDi \citep{krause-etal-2021-gedi-generative} & 1 \\
        DExperts \citep{liu-etal-2021-dexperts} & 2 \\
        RAD-V \citep{deng_2023_rad} & $k$ \\
        \mydashedcline{1-2}
        RAD-Q (Ours) & 1 \\
        \bottomrule
    \end{tabular}
    \captionof{table}{Number of input tokens a discriminator model needs to process for a single decoding step with $k$ next token candidates. All included models are based on the unidirectional Transformer \citep{vaswani_transformer} and support the caching of prefix activations.}
    \label{tab:passes_table}
  \end{minipage}
  \hfill
  \begin{minipage}{0.47\textwidth}
    \centering
    \includegraphics[width=\textwidth]{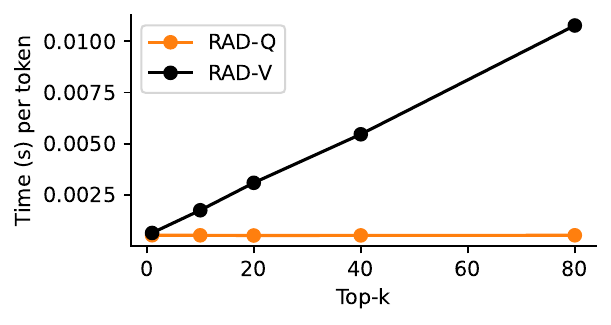}
    \caption{RAD-V processes the $k$ next token candidates separately as input requiring more time compared to RAD-Q, which relies on the output layer to obtain the scores for all next tokens.}
    \label{fig:time_per_token}
  \end{minipage}
\end{figure}

\subsection{Efficiency}
We consider using a reward model to compute the scores for $k$ candidate tokens at each of $L$ steps of decoding. Similar to RAD-V \citep{deng_2023_rad}, RAD-Q is based on the \emph{unidirectional} Transformer architecture \citep{vaswani_transformer}, which means that we can cache the prefix activations during decoding. To compute the prediction for $k$ next token candidates $v$ given a prefix $x$, RAD-V needs to pass $k$ next tokens as \emph{input} to the Transformer model, thus RAD-V processes $O(Lk)$ tokens during decoding. In contrast, RAD-Q only processes $O(L)$ tokens as input to the Transformer model and relies on the output layer to efficiently compute the scores for all next token candidates.  In \Cref{tab:passes_table}, we summarize how many tokens  external expert models process during top-$k$ decoding.
 In \Cref{fig:time_per_token}, we measure the time per generated token when running the decoding for the toxicity task with RAD-Q and RAD-V \citep{deng_2023_rad} on a single RTX A6000 GPU.

\section{Related Work}
\label{sec:related}
There are multiple approaches that investigate how to finetune a language model using attribute-conditioned data (desired/undesired examples). \citet{keskarCTRLConditionalTransformer2019} finetunes a language model using control prompts. More recent approaches \citep{schulmanProximalPolicyOptimization2017, stiennonLearningSummarizeHuman2020, luQUARKControllableText2022} perform finetuning while regularizing the weights of the model to stay close to pretrained weights. %
Despite the efficiency of decoding, these methods might require more resources for finetuning if the language model is large, or might even be unusable if we only have access to the top-$k$ logits of base language model via an API. Unlike finetuning, alternative approaches keep the language model untouched and use external models to guide the decoding from the base language model. 
\citet{dathathriPlugPlayLanguage2019} use the gradients from a discriminator to modify the prefix activations of the base model during decoding. However, gradient-based methods are costly to use during decoding since they require backpropagating through the large base model.

Closest to our work are \emph{gradient-free} guided decoding methods, where we have access to the frozen base language model and use external models to guide the sampling process from the base model. Among Q-parametrized models, GeDi \citep{krause-etal-2021-gedi-generative} uses class-conditioned language models as discriminators to augment the decoding and efficiently compute the scores for next token candidates. DExperts \citep{liu-etal-2021-dexperts} improves the quality of GeDi introducing an ensemble of two class-conditioned language models finetuned on desired and undesired data. \citet{caoSystematicRectificationLanguage2022} finetune a $Q$-style model to reduce the probability of reaching undesired terminal states. 
In concurrent work, \citet{xu2025genarm} propose a $Q$-style model, where they show that Q and V style models are equivalent, assuming the logits come from a universal approximator. While true, this analysis omits the rank bottleneck present in practical models, which we analyze in this work. We hope our work will inform this line of work of potential limitations of $Q$-style parameterized models.

$V$-style models are often used to re-rank the intermediate hypotheses \eg for best-of-n sampling, which however requires a large pool of \textit{completed} hypotheses \citep{sun2024fast}. \citet{deng_2023_rad}, \citet{sitdikovClassifiersAreBetter2022} argue to use discriminator models akin to \citet{yang-klein-2021-fudge} to guide the intermediate outputs \textit{during} the decoding, whereas \citet{deng_2023_rad} propose an effective $V$-style parametrized reward model trained from labeled data examples. \citet{sitdikovClassifiersAreBetter2022, dekoninckControlledTextGeneration2023} use available \emph{bidirectional} Transformers to guide the base language model, which, however, requires recomputing all prefix tokens at each decoding step. To tackle this issue, RAD-V \citep{deng_2023_rad} 
proposes a \emph{unidirectional} model suitable for caching of prefix activations. Among RL-based approaches, \citet{mudgal2024controlled}, \citet{chakraborty2024transfer} use a $V$-style model although at the same time they rely on the losses designed for a $Q$-style model (e.g.\ CD-Q from \citet{mudgal2024controlled}). To the best of our knowledge, there is little attention to the implied efficiency-quality trade-off that we study in our work. The closest to our analysis, \citet{han2024valueaugmentedsamplinglanguage} compare both parametrizations in relation to language tasks, where they empirically observe that $V$-style parametrization outperforms $Q$-style parametrization.

 To summarize, we complement the previous work, by zooming in into the parametrization of an autoregressive reward model. We highlight the trade-off between efficiency and expressiveness of a reward model, and showcase that 1) in theory there is a gap in expressivity between $Q$-style and $V$-style models due to rank bottleneck, 2) for the tasks and datasets we consider, higher rank-expressiveness can be traded for higher efficiency without quality drop when using distillation. We hope that our analysis will inform future work on the design choices of autoregressive reward models.

\section{Conclusion}
\label{sec:concl}

We review the recently proposed RAD-V approach of training a reward model for the guided decoding, and we reformulate it as the incomplete reward matrix learning problem. In the light of the rank analysis of the reward matrix, we observe that the high flexibility of RAD-V might not overweight its lower efficiency during decoding. We revisit the low-rank parametrization style of reward models in application to RAD, and demonstrate the effectiveness of a more efficient low-rank RAD-Q parametrization. We thus bridge the gap between two paradigms of training external expert models, demonstrating that we can have both efficient and effective controlled generation. At the same time, we precaution from indiscriminately choosing low-rank parametrization by highlighting the cases when the incomplete reward matrix has higher minimal rank.

\section*{Acknowledgments}
This publication is part of the project VI.Veni.212.228 of the research program
`Veni', which is financed by the Dutch Research Council (NWO); and is part of
‘Hybrid Intelligence: augmenting human intellect’
(https://hybrid-intelligence-centre.nl) with project number 024.004.022 of the
research program `Gravitation' which is (partly) financed by the Dutch Research
Council (NWO).

We thank Wilker Aziz, Bryan Eikema, Caio Corro, and members of LTL and CLTL for fruitful discussions and feedback. We also thank the Perspective API team for increasing the API quota for us.

\bibliography{paper}
\bibliographystyle{tmlr}

\newpage
\appendix

\begin{center}
   \Huge \bfseries Appendix
\end{center}

\textcolor{red}{Warning: (the last page of) this appendix contains model-generated text conditioned on high toxicity contexts.}

\section{Limitations}
The models discussed in this work can only reduce the probability of generating the toxic responses, not prevent it. Moreover, evaluation of toxicity is far from perfect, and even a very low toxicity score from automatic evaluation such as Perspective API does not necessary mean that the sample is `safe'. Furthermore, we should not exclusively rely on toxicity when evaluating the safety of samples from language models due to the complexity and variability of language. It is also not clear that by reducing toxicity, we are not introducing other harms.
Furthermore, both RAD-V and our models represent low-rank $\hat{R}$ and further qualitative research is needed to investigate whether certain toxicity patterns require high rank to represent them.

\section{Reward Matrix}
\label{app:loss_minimizer}

To train a reward model, we use weighted mean squared loss, for which the weighted mean recovers the minimum:

\begin{equation} r^* = \argmin_{r} \sum_{\lambda, y} \lambda (r - y)^2 = \frac{\sum_{\lambda, y}\lambda y}{\sum_{\lambda,y} \lambda}
\end{equation}
\begin{proof}
    $\frac{\partial}{\partial r} \sum_{\lambda, y} \lambda (r - y)^2 = \sum_{\lambda, y} \frac{\partial}{\partial r}[\lambda (r - y)^2 ] =  \sum_{\lambda, y} 2\lambda(r-y) = 2 (r \sum_{\lambda,y} [\lambda] - \sum_{\lambda, y}[\lambda y]) = 0$.
    Hence, $r^* = \frac{\sum_{\lambda, y}\lambda y}{\sum_{\lambda,y} \lambda}$
\end{proof}

\section{Factorization of $P_{\Omega}(R)$}

Any matrix $R \in \mathbb{R}^{N \times |V|}$ can be factored as $R=UV^T$ with $U, V$ of dimensions $N \times q; |V| \times q$. If $R$ is \emph{incomplete}, then there are in general multiple possible factorizations of $P_{\Omega}(R)$ compatible with the observed values.

\subsection{Rank-$1$ case}
\label{app:rank_1}
To get a better intuition why the incompleteness of $P_{\Omega}(R)$ allows us to find a compatible factorization with lower minimal rank, consider a simple example. If we only know $1$ element per row of $R$, then the minimal rank of $P_{\Omega}(R)$ is equal to $1$. To prove this, consider completing $P_{\Omega}(R)$ such that each row is filled with the same element (the only one known for this row):

\begin{equation*}
\begin{pmatrix}
1 & ? & ? \\
? & 4 & ? \\
? & ? & 3 \\
\end{pmatrix} 
\rightarrow
\begin{pmatrix}
1 & 1 & 1 \\
4 & 4 & 4 \\
3 & 3 & 3 \\
\end{pmatrix} 
\end{equation*}

\subsection{A case of a single missing value.}
\label{app:lemma_1}
Here we prove that introducing a missing value in a random $k$-by-$k$ matrix will almost always (with probability 1) allow a completion of rank at most $k-1$.

\begin{proof}
    Let $R$ be a $k$-by-$k$ matrix, and
    without loss of generality, by permuting rows and columns we may assume the unknown element is at position \((1,1)\).
    Consider completing it with a value $x$, and denote the completed matrix as $A$.
    Let $A_{ij}$ be the $ij$ minor of $A$, \ie the determinant of the submatrix left after removing row $i$ and column $j$.

    We can expand $\det(A) = x A_{11} + a_{21} A_{21} + \ldots + a_{k1} A_{k1}$. In order for $A$ to not be full rank, there needs to be a solution in $x$ to $\det(A)=0$. If $A_{11}$ is not $0$, then $x = -(a_{21} A_{21} + \ldots + a_{k1} A_{k1}) / A_{11}$ will lead to a completion of rank less than $k$
    regardless of any of the specific values.
    However, if $A_{11}$ is $0$ but $a_{21} A_{21} + \ldots + a_{k1} A_{k1}$ is not $0$, then there is no solution for $x$ that would result in a determinant of $0$ 
    and thus the rank of $A$ must be $k$.
    The set of matrices $R$ satisfying $R_{11} = 0$ has zero Lebesgue measure inside $\bbR^{k \times k}$, and thus
    for any distribution over matrices with full support, sampling such $R$ is a zero-probability random event.%
\end{proof}

\subsection{Estimating the minimal rank of the data}
\label{sec:is_data_low_rank}

Empirically calculating minimal rank is challenging due to the very large number of prefixes (row of the matrix), particularly, a large portion of the prefixes have unique continuations. We show how we simplify the minimal rank estimation by considering only the prefixes with two or more continuations, and demonstrate that partially observed $\hat{R}$ can be fit with a low rank matrix factorization.

 Given a training dataset of responses and text utterances, there will be many unique prefixes, for which the \Cref{example:1} is applicable. We can therefore reduce the complexity of rank estimation by skipping these prefixes, as shown in the next result.
 \begin{lemma}
 \label{lemma:rank_block}
Let \((R, \Omega)\) be a partially-observed matrix and  \(\Omega_2 \subseteq \Omega\) denote the subset of observed indices that have at least one other index in the same row. 
Then,
\begin{equation}
\minepsrank{R} \le 1 + 
\operatorname{min~rank}_{{\Omega_2}, \varepsilon}(R)
\end{equation}
\end{lemma}
\proof{Permute rows such that all rows with a single index are grouped together. We have a block-incomplete matrix where the top block admits a completion of rank 1, and the bottom block admits a completion of rank 
$\operatorname{min~rank}_{{\Omega_2}, \varepsilon}(R)$.
The rank of the stacked completions is no more than the sum of the ranks of the completions.}

This result reduces the problem of estimating
$\minepsrank{R}$ to the possibly smaller problem of estimating
$\operatorname{min~rank}_{{\Omega_2}, \varepsilon}(R)$
(since any fully-unobserved rows and columns can be skipped.)

We now numerically verify 
that there exist low-rank factorizations
compatible with $P_{\Omega_2}(R)$ within $\varepsilon$.
Finding such a factorization of rank $r-1$ implies, by the previous lemma, that $R$ is of minimal rank $r$ \wrt $\Omega$ and thus
that the training dataset $D_f$ can be fit by a reward model with rank bound by $r$, regardless of the specifics of said model.
In general, finding minimal rank factorization of incomplete matrices is known to be NP-hard, and usually convex relaxation such as minimization of the nuclear norm is considered (see \citet{LowRankMatrix2009}). To 
factorize a partially-observed matrix, we use the \emph{soft impute} alternating least squares algorithm \citep{JMLR:v11:mazumder10a, hastie2014matrixcompletionlowranksvd} \footnote{\url{https://cran.r-project.org/web/packages/softImpute/index.html}}. 
Given a matrix $X \in \bbR^{n \times m}$ with observed indices \(\Omega\), this algorithm solves
\begin{equation}
\begin{aligned}
    \operatorname{minimize}~& \left\| P_\Omega(X - AB^\top) \right\|^2_F + \lambda(\|A\|^2_F + \|B\|^2_F) \\
    \text{with respect to}~& A \in \bbR^{n \times k}, B \in \bbR^{m \times k}\\
\end{aligned} 
\end{equation}
by alternating between efficient soft SVD updates for 
$A$ given $B$ and $B$ given $A$.
We optimize for 1000 iterations with a 
trace norm penalty of $\lambda=10^{-4}$.
This penalty induces a small bias but improves convergence. At convergence, $A$ and $B$ constitute a \emph{certificate} of the minimal numerical rank of $X$ \wrt $\Omega$.

Results reported in \cref{tab:mselowrank}
provide strong evidence that 
both datasets can be approximated well by low rank matrices, close to the $10^{-6}$ resolution of 32-bit floating point numbers. 
The reported MSE values for rank $d-1$ can be interpreted as reachable lower bounds of the MSE training loss of a RAD-Q transformer with hidden dimension $d$ on the respective training data.

\begin{table}
\begin{center}
\begin{tabular}{l l l l l}
\toprule
rank & detox. & sentiment & helpfulness & safety \\
\midrule
0 & 4.6e-2 & 4.9e-1   &  8.4  &   0.42  \\
255 & 3.4e-5 & 3.4e-4 &  1.04e-08  &  5.08e-08   \\
511 & 7.7e-6 & 9.7e-5 &  1.45e-09  &  5.77e-10   \\
767 & 4.7e-7 & 6.6e-7 &  1.28e-09  &  4.81e-10    \\
\bottomrule
\end{tabular}
\end{center}
\caption{Mean squared errors of low-rank matrix completion of the Jigsaw (detox) and Amazon review polarity (sentiment) datasets following the methodology described in \cref{sec:is_data_low_rank}. Additionally, we report MSE for HelpSteer (helpfulness) dataset \citep{wang-etal-2024-helpsteer} and BeaverTails (safety) dataset \citep{NEURIPS2023_beavertails}, which are commonly used for reward model training \citep{wang2024interpretablepreferencesmultiobjectivereward}. The zero rank row corresponds to predicting the zero matrix. 
All datasets can be approximated well by low rank models. For the ranks, we use multiples of $256-1$, because one rank is reserved to handle the single-occurrence contexts.}
\label{tab:mselowrank}
\end{table}

\section{Rank expressivity of $\hat{R}_{\text{RAD-V} }$ and $\hat{R}_{\text{RAD-Q}}$}
\label{exp:low_rank_fails}

\subsection{Rank expressivity of RAD-$V$}
\label{app:rad_rank}
In this experiment, we empirically verify that RAD-V is capable to approximate $P_{\Omega}(R)$ matrix with $\rank(P_{\Omega}(R))>d$, where $d$ is the dimensionality of the model. We finetune RAD-V initialized from the GPT-2-Small (with $d=764$) on a synthetic data constructed as in \cref{example:3}. We generate $P_\Omega(R), n=1024>d$, an incomplete matrix of size $1024$ with unknown elements above the diagonal, ones on the diagonal, and zeros below the diagonal. 
With this construction,
$\minrank{R}=n$ and thus greater than the model dimension.

We verify that we can train RAD-V to fit this train matrix obtaining the training loss (MSE) less than $5\cdot10^{-6}$.

\subsection{Rank expressivity of RAD-$Q$}

RAD-Q approximates $P_{\Omega}(R)$ as a product of two rank $d$ matrices, hence RAD-Q cannot reconstruct the data perfectly, which has higher rank; RAD-Q obtains MSE $>0.0001$.

 We thus conclude that RAD-V (in contrast to RAD-Q) is indeed capable of representing $P_{\Omega}(R)$ matrices with a rank higher than $d$.

 \subsection{Real data experiments} For the experiment with the real datasets for the detoxification and sentiment control tasks, in \Cref{fig:rank}, we numerically measure the rank of $R_{RAD-V} $ and $R_{RAD-Q}$, and observe that both RAD-Q and RAD-V learn low-rank reward matrices. We thus conclude that both these models have needed capacity to represent the incomplete $P_{\Omega}(R)$ matrices obtained from the datasets.

\subsection{Numerical rank}
\label{app:matrix_rank}
To compute rank of $n \times m$ matrix, we use the default cutoff in Numpy \footnote{\url{https://numpy.org/doc/stable/reference/generated/numpy.linalg.matrix_rank.html}} and PyTorch \footnote{\url{https://pytorch.org/docs/stable/generated/torch.linalg.matrix_rank.html}} at the time of writing, which is to say we count only singular values above $\max(m,n) \varepsilon \sigma_1$, where $\varepsilon$ is the machine epsilon for the corresponding data type, i.e., the difference between $1.0$ and the next smallest representable number larger than $1.0$, and $\sigma_1$ is the largest singular value.

There are potential issues that may arise when computing the numerical rank. One issue is that the singular values, especially for the matrices coming from 32bit float precision neural network, will not be exactly zero, so this is why libraries like Numpy or PyTorch use a precision-based cutoff for singular values that should be considered indistinguishable from zero; we use the default such parameters. The other issue is that the number of rows in the reward matrices is very high and we follow the work of \citet{finlayson2024closing} and estimate rank by sampling rows from the matrix. Different submatrices can have different ranks, but we sample i.i.d. to prevent this.

\begin{figure}[tbp]
  \centering
  \subfloat[Detoxification task.]{
    \includegraphics[width=0.47\textwidth]{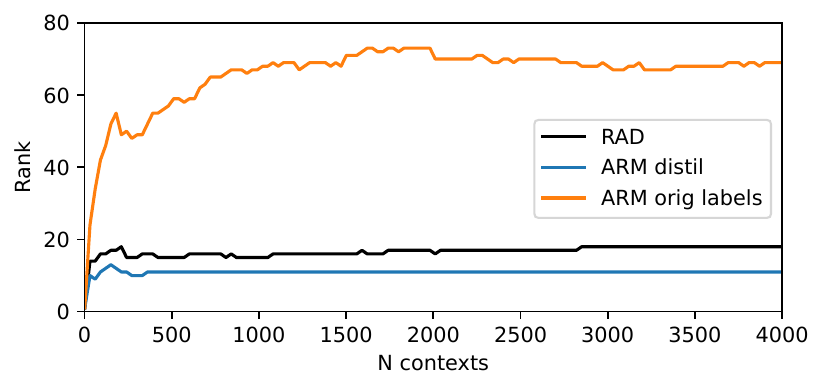}
  }
 \hfill
 \subfloat[Sentiment task.]{
    \includegraphics[width=0.47\textwidth]{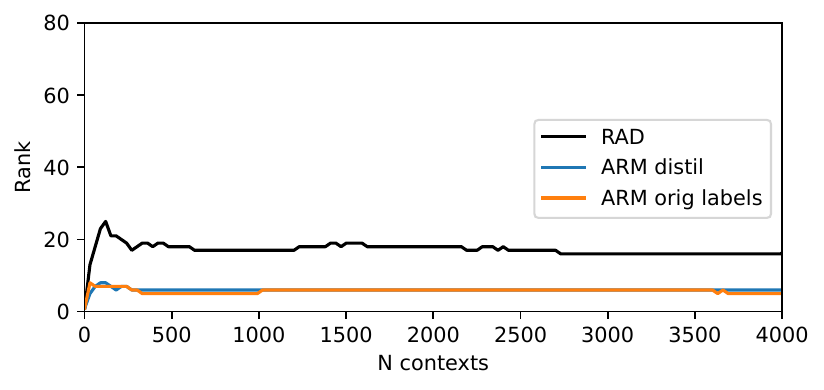}
}
 \caption{We numerically estimate the ranks of both $\hat{R}_{\text{RAD-V} }$ and $\hat{R}_{\text{RAD-Q}}$ increasing the number of training prefixes (rows of $\hat{R}$). In all cases, the ranks tend to be less than the model dimension $d=764$. This means that rank-capacity of RAD-Q is sufficient to capture the training datasets for the detoxification and sentiment tasks.}
  \label{fig:rank}
\end{figure}

\section{Training Details}
\label{app:training_details}
To train reward models, we reuse the hyperparameters from \citet{deng_2023_rad}, where possible. We finetune the reward models with Adam optimizer \citep{kingma:adam} with $\beta_1=0.9, \beta_2=0.95, \epsilon=1\mathrm{e}{-12}$. We use weight decay $0.01$, batch size $100$, and the learning rate changes linearly from the initial value ($10^{-5}$ by default) to zero.

To train RAD-Q, we initialize the parameters with the pretrained GPT-2-Small/TinyLLaMa \footnote{\url{https://huggingface.co/TinyLlama/TinyLlama-1.1B-intermediate-step-1431k-3T}} weights, and freeze the shared input-output embedding parameters. Alternative strategy would be to use parameter efficient finetuning \citep{hu2022lora, sidahmed2024parameterefficientreinforcementlearning}.

\subsection{Detoxification}
For the detoxification task, we finetune RAD-Q with the learning rate $10^{-5}$ for $5$ epochs.

For the LLaMa-2, we additionally finetune RAD-V with the TinyLLaMa backbone for the fair comparison with RAD-Q.

\subsection{Sentiment Control}
To finetune RAD-Q on responses only for sentiment control task, we first finetune the model with the learning rate $10^{-5}$ on the Amazon Polarity dataset, and then finetune it for $5$ epochs on the SST-2 dataset with the learning rate $2\mathrm{e}{-6}$. For distillation experiment, we finetune RAD-Q for $5$ epochs with the learning rate $10^{-5}$ on Amazon Polarity dataset.

\section{MAUVE}
\label{app:mauve}
To complement perplexity as a measure of fluency, we use MAUVE \citep{pillutla2021mauve} as one of the fluency metrics. For  reference texts, we take the generations of unguided model (GPT-2, or LLaMa-2-(7b/13b). Thus, this metric should capture how close the distribution of the continuations of a guided model is to the distribution of the original language model.
To calculate MAUVE, we follow recommendations of \citet{He_Zhang_Wang_Kumar_Cho_Glass_Tsvetkov_2023} and use ELECTRA-large model to obtain the text representations. We use the hyperparameters of \citet{pillutla2021mauve}: $c=5$ for the scaling constant; $k-$means for the quantization algorithm with $500$ iterations, and $n/10$ clusters where $n$ is the number of generations. To compute MAUVE, we use $1000$ prompts from the evaluation dataset.

\section{Results}

\subsection{Detoxification}
\subsubsection{Results with Perspective API classifier}
\label{app:detox_perspective}

In this section, we report full results with the Perspective API as a toxicity classifier.
\paragraph{GPT-2.}
\label{sec:app_detox}
Results for the detoxification task with the GPT-2-Large base model and GPT-2-small reward model, are presented in \Cref{tab:full_toxicity}. 

We present the results for RAD-Q and RAD-V with \emph{top-k} decoding with $k=40$ in \Cref{fig:toxicity_k_40}. We observe similar relative performance of RAD-Q compared to RAD-V as in the experiment with $k=20$, presented in the \Cref{fig:toxicity}.

\paragraph{LLaMa-2.}
Results for detoxification task with LLaMa-2-(7b/13b) base model and TinyLLaMa reward model are presented in \Cref{fig:llama_toxicity_perspective} and \Cref{tab:llama_toxicity}.

\paragraph{Baselines.} Additionally, in \Cref{fig:deng_et_al_toxicity_results}, we include results from \cite{deng_2023_rad} for other baseline models (for an older version of Perspective API).

To highlight the difference between the RAD-Q and DExperts, we show the trade-off plot for DExperts model in \Cref{fig:comparison_to_dexperts}, varying the $\alpha$ scalar parameter for DExperts. As we can observe, the RAD-Q has better constraint satisfaction / fluency trade-off than DExperts model. We attribute this to the difference in the training objectives of the expert models (reward modeling or language modeling), as argued in \citep{deng_2023_rad}.

\begin{figure}
    \centering
    \includegraphics[width=0.7\textwidth]{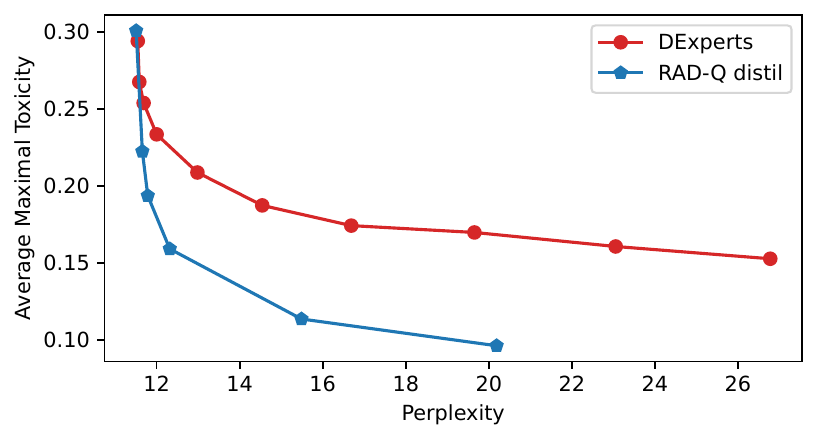}
    \caption{Comparison of toxicity/fluency trade-off between RAD-Q (distill) and DExperts. We rerun the sampling from these two models using \emph{top-k} decoding with $k=20$. Results are calculated over randomly selected $1000$ prompts. We observe, that RAD-Q show better constraint satisfaction/fluency than DExperts.}
    \label{fig:comparison_to_dexperts}
\end{figure}

\begin{figure}
    \centering
    \includegraphics[width=0.7\textwidth]{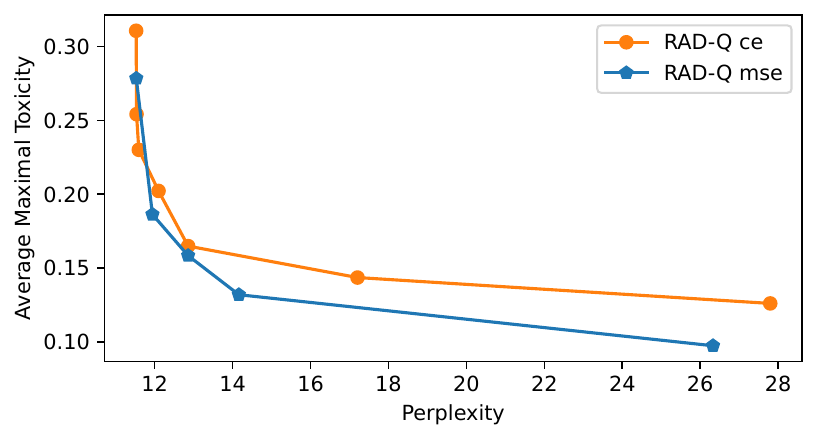}
    \caption{Comparison of RAD-Q model trained on original responses with squared loss vs with cross-entropy loss. We rerun the sampling from these two models using \emph{top-k} decoding with $k=20$. Results are calculated over randomly selected $1000$ prompts. We observe, that RAD-Q trained with the squared loss show slightly better constraint satisfaction/fluency than RAD-Q trained with cross-entropy loss.}
    \label{fig:ce_vs_mse_loss}
\end{figure}

\begin{figure}
    \centering
    \includegraphics[width=0.7\textwidth]{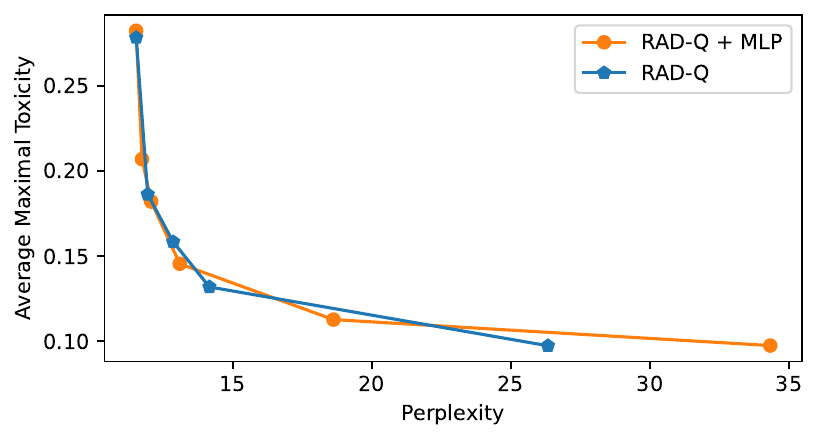}
    \caption{Comparison of RAD-Q trained on original responses with linear parametrization vs with non-linear MLP parametrization. We rerun the sampling from these two models using \emph{top-k} decoding with $k=20$. Results are calculated over randomly selected $1000$ prompts. We observe, that both parametrizations perform very closely.}
    \label{fig:mlp}
\end{figure}

\subsection{Additional Ablation Results}

\subsubsection{Loss choice}
We perform an ablation study for the choice of the loss function used to train RAD-Q \Cref{eq:cumulative_loss}. There, we follow the approach of \citep{deng_2023_rad}, where they introduce the squared loss (see section 2.1 Unidirectional Reward Model). An alternative strategy would be to use the binary cross-entropy loss, using the fact that for our datasets responses $y$ are from $[0;1]$ range:
\begin{equation}
\label{eq:cumulative_loss_ce}
    \cL_{ce}(\hat{r}(v|x'), y, \lambda) = \lambda (y \log \sigma(\hat{r}(v|x'))  + (1-y)\log(1-\sigma(\hat{r}(v|x'))) ),
\end{equation}
where we introduced $\sigma(x)=1/(1 + e^{-x})$ function to softly map the predictions of RAD-Q into $[0;1]$ range, which is also used during generation. \Cref{fig:ce_vs_mse_loss} demonstrates that the RAD-Q trained with the squared loss slightly outperform the RAD-Q trained with the binary cross-entropy loss.

\subsubsection{MLP vs Linear Parametrization}

In this ablation, we consider replacing the linear parametrization of RAD-Q  \Cref{eq:rad_factors} with a non-linear MLP parametrization:
\begin{equation}
    \Delta\hat{r}_{\text{RAD-Q+MLP}}(x) := W_1 \sigma (W_2 E^T W^T h(x)^T),
\end{equation}
where $W_1 \in \mathbb{R}^{d \times |V|}; W_2 \in \mathbb{R}^{|V| \times d}$. As we observe in \Cref{fig:mlp}, MLP parametrization performs on par with the linear parametrization. We thus recommend using a more simple linear parametrization.

\subsection{Sentiment control}

Here, in \Cref{fig:sentiment_top_k_40}, we include the additional results for the RAD-V and RAD-Q with \emph{top-k} decoding with $k=40$.

\subsubsection{Results with RoBERTa classifier}
\label{app:roberta}
In addition to toxicity scores with Perspective API, we provide the results with the open-weight RoBERTa toxicity classifier \citep{nicholas22aira} for the guided generation with GPT-2 (\Cref{fig:reproducible_gpt_perplexity} and \Cref{tab:full_toxicity_roberta}) and the LLaMa-2 (\Cref{fig:reproducible_llama_perplexity} and \Cref{tab:llama_toxicity_roberta}). We notice that results for the average maximal toxicity with RoBERTa are relatively similar to the result with Perspective API. We hope that with an open-weight classifier it will be easier for the community to directly compare to the published results without the need to recompute the API scores.

\begin{figure}[ht]
  \centering
    \centering
    \includegraphics[width=\textwidth]{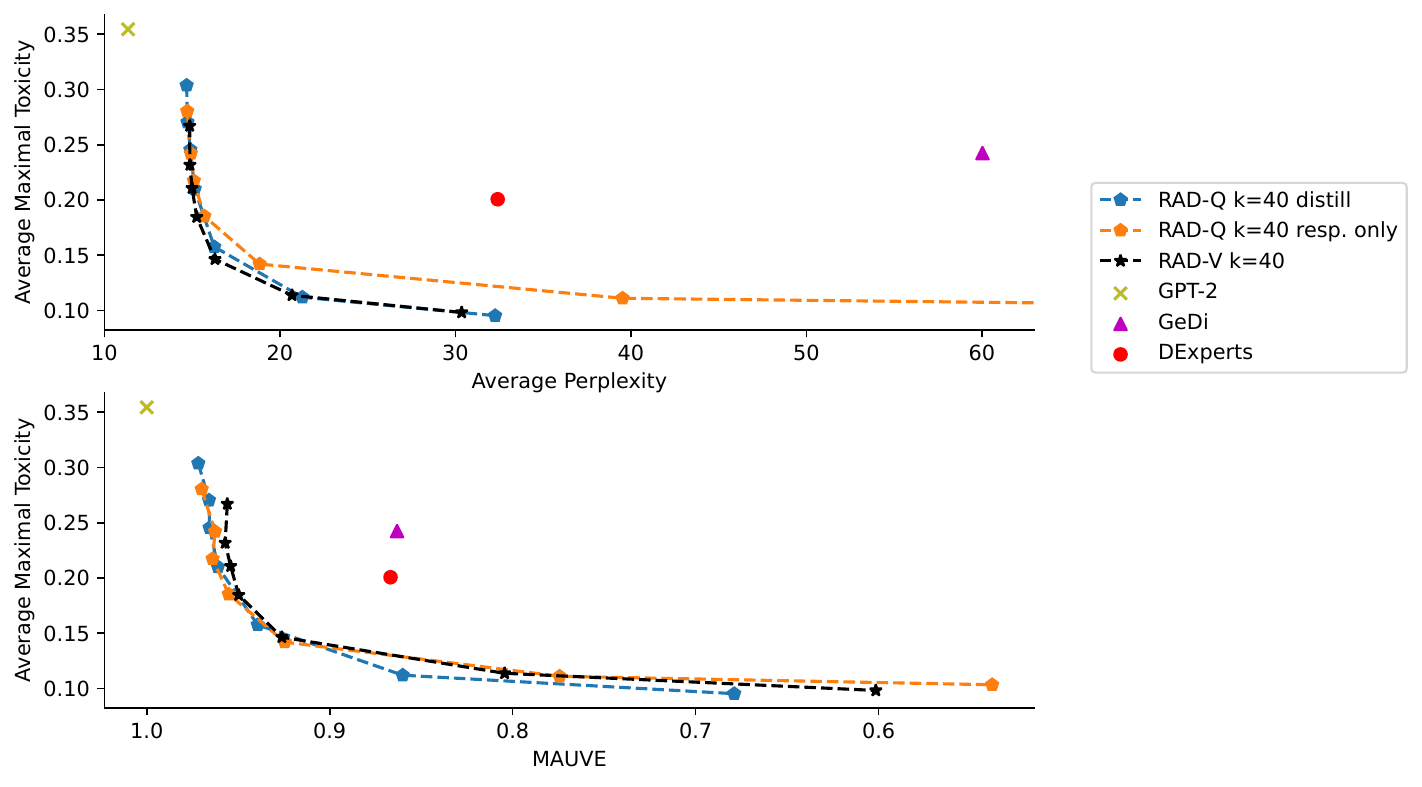}
    \caption{Additional results for detoxification task for RAD-Q and RAD-V with $k=40$.}
    \label{fig:toxicity_k_40}
  
\end{figure}

\begin{figure}
    \centering
    \includegraphics[width=\textwidth]{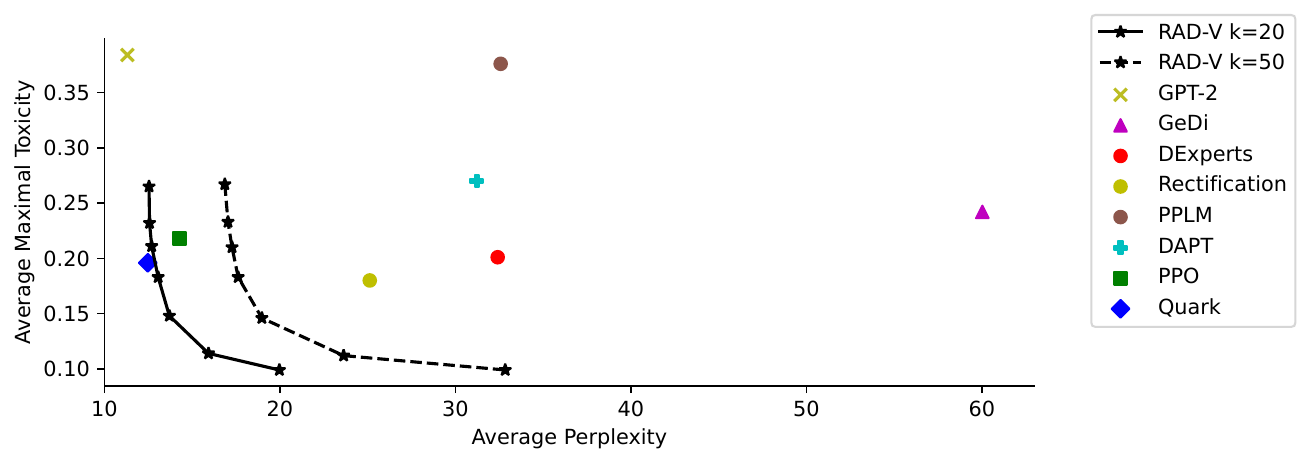}
    \caption{Detoxification results reported in \citet{deng_2023_rad} with Perspective API with GPT-2-Large model (API queries made between May and June 2023).}
    \label{fig:deng_et_al_toxicity_results}
\end{figure}

\begin{figure}[t]
  \centering
  \begin{minipage}{0.9\textwidth}
  \includegraphics[width=\textwidth]{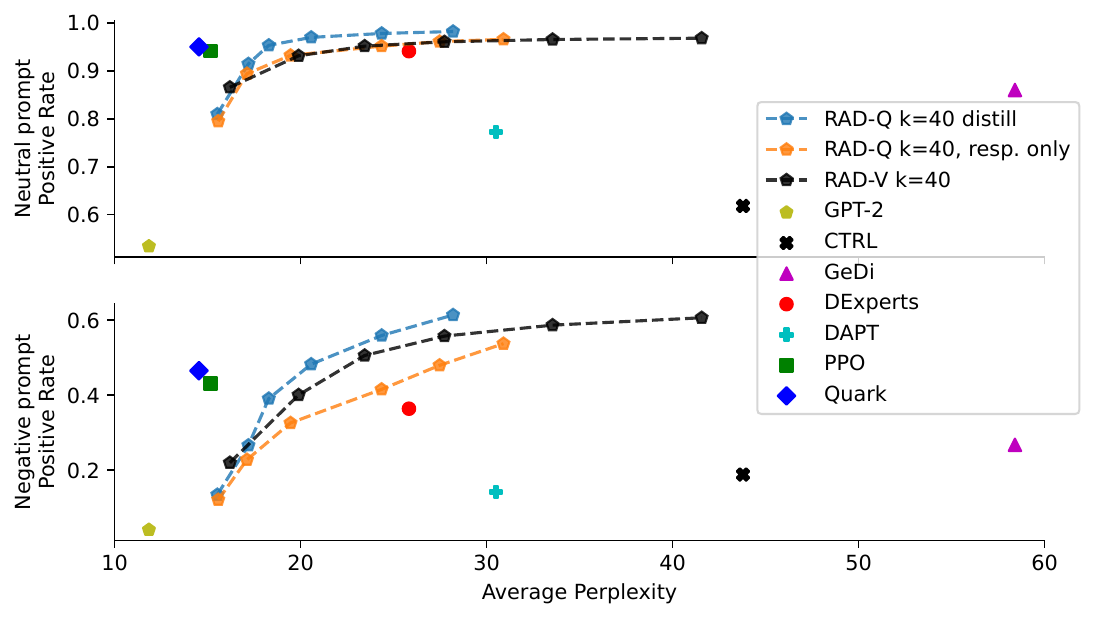}
  \label{fig:sentiment_k_40}
  \end{minipage}

  \centering
  \begin{minipage}{0.9\textwidth}
  \includegraphics[width=\textwidth]{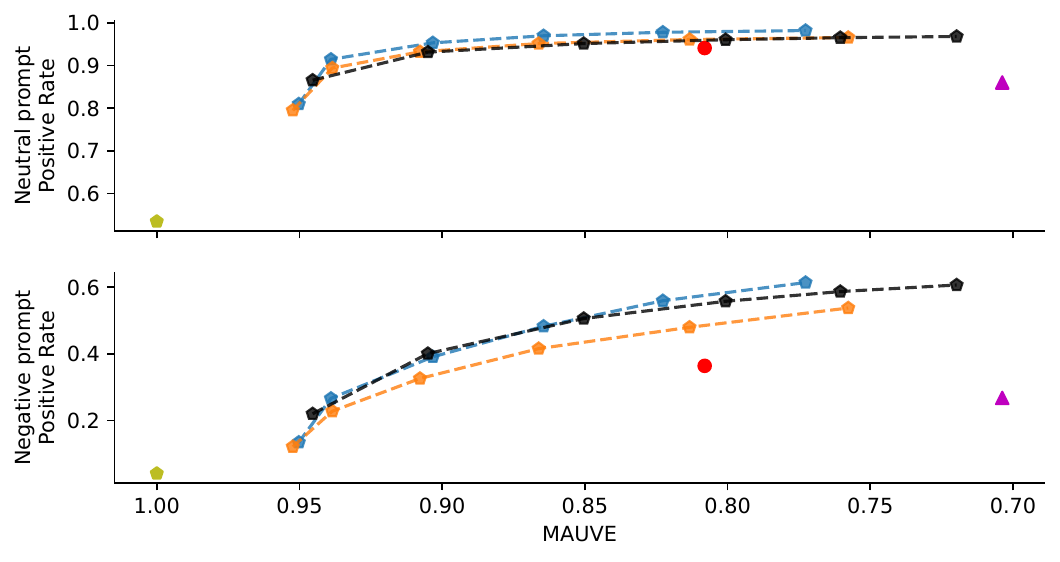}
  \caption{Additional results for sentiment control task with $k=40$. For this plot with average perplexity, we include the results from \cite{deng_2023_rad} for other baselines for reference.}
  \label{fig:sentiment_top_k_40}
  \end{minipage}
\end{figure}

\begin{figure}
    \centering
    \includegraphics[width=\textwidth]{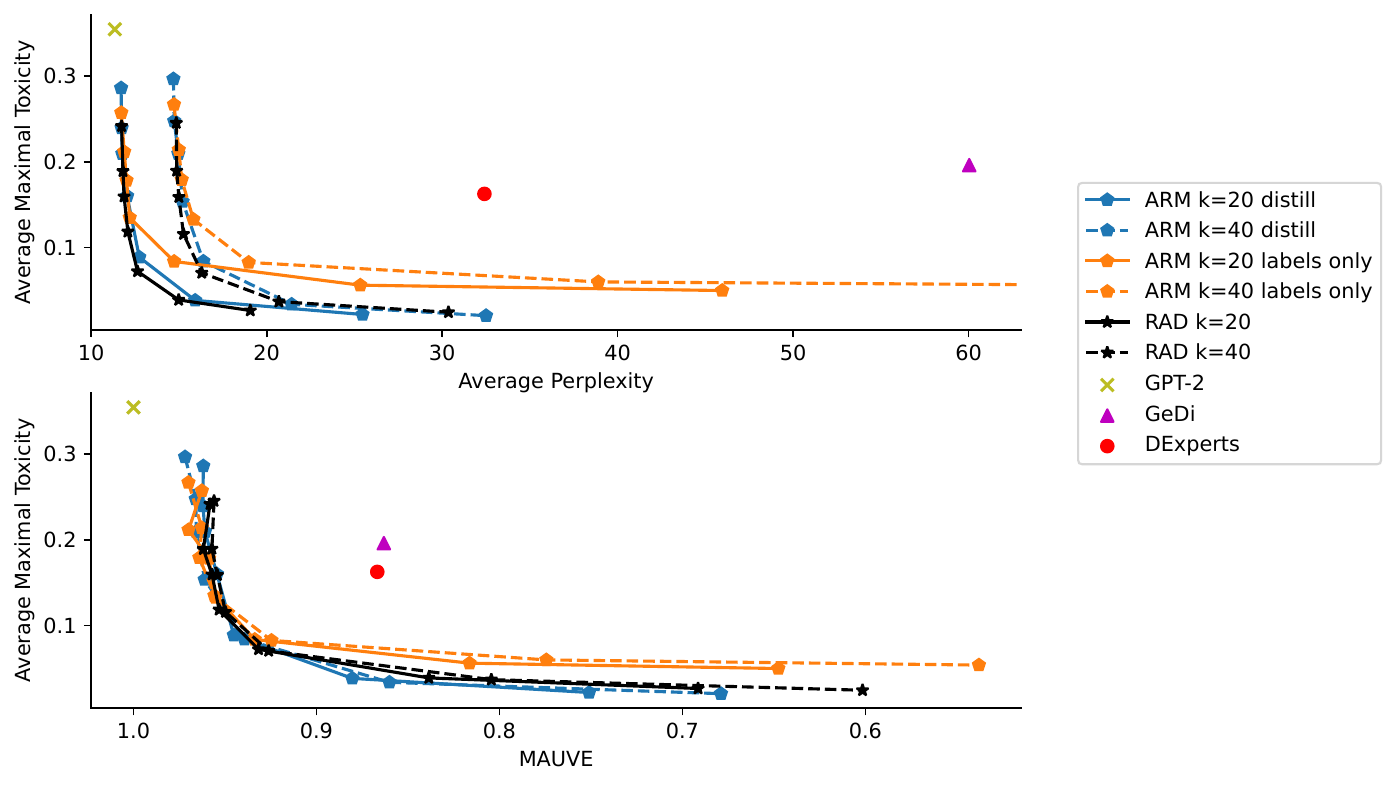}
    \caption{Detoxification results with finetuned RoBERTa toxicity classifier \citep{nicholas22aira} and the GPT-2-Large base model.}
    \label{fig:reproducible_gpt_perplexity}
\end{figure}

 \begin{figure}
    \centering
    \includegraphics[width=\textwidth]{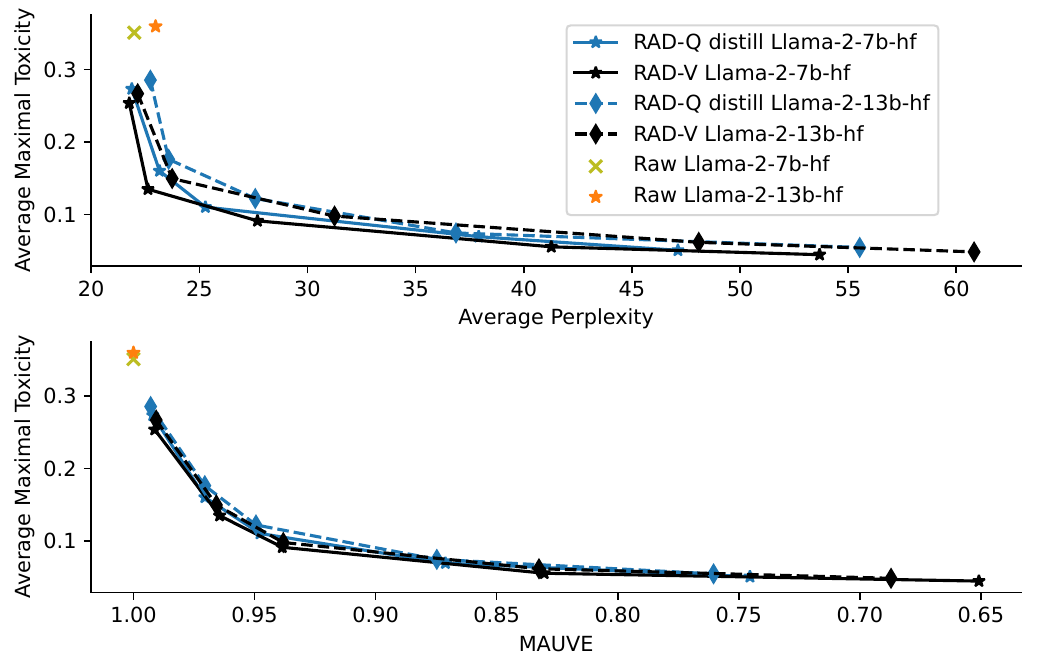}
    \caption{Detoxification results with finetuned RoBERTa toxicity classifier \citep{nicholas22aira} and the LLaMa-2 family of models.}
    \label{fig:reproducible_llama_perplexity}
\end{figure}

\begin{table}[ht]
\begin{tabular}{crccccccc}
\toprule

& & & \multicolumn{2}{c}{\% Toxicity ($\downarrow$)} & \multicolumn{2}{c}{Fluency} & \multicolumn{2}{c}{Diversity ($\uparrow$)} \\[0.5ex]
Model & & $\beta$ & Avg. Max & Toxic  & PPL ($\downarrow$) & MAUVE ($\uparrow$) & Dist 2 & Dist 3 \\

 & & $\beta$ & Toxicity & Rate & & & & \\

\midrule
\multirow{14}{*}{RAD-Q distill} &
\multirow{7}{*}{\emph{k=20}} &
    10  & 0.301 & 0.139 & 11.70 & 0.96 & 0.81 & 0.84 \\
& & 20  & 0.270 & 0.096 & 11.73 & 0.96 & 0.81 & 0.84 \\
& & 30  & 0.246 & 0.071 & 11.77 & 0.96 & 0.81 & 0.84 \\
& & 50  & 0.212 & 0.043 & 11.98 & 0.95 & 0.81 & 0.84 \\
& & 100 & 0.160 & 0.019 & 12.67 & 0.95 & 0.80 & 0.83 \\
& & 200 & 0.117 & 0.005 & 15.78 & 0.88 & 0.78 & 0.81 \\
& & 300 & 0.097 & 0.002 & 24.53 & 0.75 & 0.75 & 0.79 \\

\mydashedcline{2-8}
& \multirow{7}{*}{\emph{k=40}} &
    10  & 0.304 & 0.137 & 14.68 & 0.97 & 0.83 & 0.85 \\
& & 20  & 0.270 & 0.092 & 14.73 & 0.97 & 0.83 & 0.85 \\
& & 30  & 0.245 & 0.064 & 14.90 & 0.97 & 0.83 & 0.85 \\
& & 50  & 0.210 & 0.039 & 15.14 & 0.96 & 0.83 & 0.85 \\
& & 100 & 0.158 & 0.013 & 16.26 & 0.94 & 0.83 & 0.84 \\
& & 200 & 0.112 & 0.003 & 21.28 & 0.86 & 0.81 & 0.83 \\
& & 300 & 0.095 & 0.002 & 32.27 & 0.68 & 0.78 & 0.80 \\

\myhline
\multirow{14}{*}{RAD-Q resp. only} &
\multirow{7}{*}{\emph{k=20}} &
    10  & 0.278 & 0.097 & 11.71 &  0.96 & 0.81 & 0.84 \\
& & 20  & 0.241 & 0.053 & 11.81 &  0.97 & 0.81 & 0.84 \\
& & 30  & 0.218 & 0.029 & 12.02 &  0.96 & 0.81 & 0.84 \\
& & 50  & 0.185 & 0.014 & 12.26 &  0.96 & 0.81 & 0.84 \\
& & 100 & 0.143 & 0.004 & 14.79 &  0.93 & 0.80 & 0.83 \\
& & 200 & 0.113 & 0.002 & 25.31 &  0.82 & 0.76 & 0.79 \\
& & 300 & 0.102 & 0.002 & 45.82 &  0.65 & 0.72 & 0.75 \\

\mydashedcline{2-8}
& \multirow{7}{*}{\emph{k=40}} &
    10  & 0.280 & 0.091 & 14.72 & 0.97 & 0.83 & 0.85 \\
& & 20  & 0.242 & 0.046 & 14.92 & 0.96 & 0.83 & 0.85 \\
& & 30  & 0.217 & 0.028 & 15.09 & 0.96 & 0.83 & 0.85 \\
& & 50  & 0.185 & 0.013 & 15.69 & 0.96 & 0.83 & 0.85 \\
& & 100 & 0.142 & 0.003 & 18.84 & 0.92 & 0.82 & 0.84 \\
& & 200 & 0.111 & 0.002 & 39.53 & 0.77 & 0.79 & 0.80 \\
& & 300 & 0.103 & 0.002 & 83.36 & 0.54 & 0.74 & 0.76 \\

\myhline
\multirow{14}{*}{RAD-V}  &
\multirow{7}{*}{\emph{k=20}} &
    10  & 0.265 & 0.077 & 11.73 & 0.96 & 0.81 & 0.84 \\
& & 20  & 0.231 & 0.040 & 11.81 & 0.96 & 0.81 & 0.84 \\
& & 30  & 0.211 & 0.024 & 11.87 & 0.96 & 0.81 & 0.84 \\
& & 50  & 0.184 & 0.014 & 12.09 & 0.95 & 0.81 & 0.84 \\
& & 100 & 0.149 & 0.005 & 12.64 & 0.93 & 0.81 & 0.83 \\
& & 200 & 0.115 & 0.002 & 14.98 & 0.84 & 0.79 & 0.81 \\
& & 300 & 0.099 & 0.001 & 19.08 & 0.69 & 0.76 & 0.78 \\

\mydashedcline{2-8}
& \multirow{7}{*}{\emph{k=40}} &
    10  & 0.267 & 0.072 & 14.86 & 0.96 & 0.83 & 0.85 \\
& & 20  & 0.232 & 0.036 & 14.87 & 0.96 & 0.83 & 0.85 \\
& & 30  & 0.211 & 0.021 & 14.99 & 0.95 & 0.83 & 0.85 \\
& & 50  & 0.185 & 0.011 & 15.26 & 0.95 & 0.83 & 0.85 \\
& & 100 & 0.146 & 0.005 & 16.30 & 0.93 & 0.83 & 0.84 \\
& & 200 & 0.114 & 0.002 & 20.69 & 0.80 & 0.82 & 0.83 \\
& & 300 & 0.098 & 0.001 & 30.36 & 0.60 & 0.79 & 0.80 \\

\bottomrule
\end{tabular}
\centering
\caption{Results for detoxification task with the Perspective API as a toxicity classifier. Calls to the Perspective API were performed in June-July 2024.}
\label{tab:full_toxicity}
\end{table}

\subsection{Sentiment Control}

Results for sentiment control task with the GPT-2-Large are presented in \Cref{tab:full_sentiment}.

\begin{table}[ht]
\begin{tabular}{ccccccccc}
\toprule
    & & & \multicolumn{2}{c}{\% Positive Rate ($\uparrow$)} & \multicolumn{2}{c}{Fluency} & \multicolumn{2}{c}{Diversity ($\uparrow$)} \\[0.5ex]
Model &  &  $\beta$ & Negative & Neutral & PPL ($\downarrow$) & MAUVE ($\uparrow$) & Dist 2 & Dist 3 \\
\midrule
\multirow{12}{*}{RAD-Q distill} &
\multirow{6}{*}{\emph{k=20}} &

    10 & 12.94 & 81.08 & 12.16 & 0.96 & 0.76 & 0.78 \\
& & 20 & 24.87 & 91.00 & 12.85 & 0.94 & 0.75 & 0.78 \\
& & 30 & 35.18 & 94.87 & 14.11 & 0.92 & 0.75 & 0.78 \\
& & 40 & 43.60 & 96.60 & 15.74 & 0.89 & 0.75 & 0.78 \\
& & 50 & 49.84 & 97.38 & 18.03 & 0.86 & 0.74 & 0.78 \\
& & 60 & 55.34 & 97.87 & 20.09 & 0.81 & 0.73 & 0.77 \\
\mydashedcline{2-9}
& \multirow{6}{*}{\emph{k=40}} &
    10 & 13.50 & 80.97 & 15.53 & 0.95 & 0.78 & 0.79 \\
& & 20 & 26.66 & 91.45 & 17.20 & 0.94 & 0.78 & 0.79 \\
& & 30 & 39.12 & 95.32 & 18.29 & 0.90 & 0.78 & 0.80 \\
& & 40 & 48.28 & 96.98 & 20.57 & 0.86 & 0.77 & 0.79 \\
& & 50 & 55.94 & 97.80 & 24.36 & 0.82 & 0.76 & 0.79 \\
& & 60 & 61.39 & 98.21 & 28.20 & 0.77 & 0.75 & 0.78 \\

\myhline
\multirow{12}{*}{RAD-Q resp. only} &
\multirow{6}{*}{\emph{k=20}} &
    10 & 12.13 & 80.02 & 12.19 & 0.96 & 0.75 & 0.78 \\
& & 20 & 21.24 & 89.06 & 13.67 & 0.95 & 0.75 & 0.78 \\
& & 30 & 29.94 & 92.66 & 15.29 & 0.93 & 0.74 & 0.78 \\
& & 40 & 37.38 & 94.62 & 17.06 & 0.89 & 0.74 & 0.78 \\
& & 50 & 43.19 & 95.65 & 20.11 & 0.85 & 0.72 & 0.77 \\
& & 60 & 47.19 & 96.20 & 23.07 & 0.82 & 0.71 & 0.76 \\

\mydashedcline{2-9}
& \multirow{6}{*}{\emph{k=40}} &
    10 & 12.17 & 79.49 & 15.58 & 0.95 & 0.78 & 0.79 \\
& & 20 & 22.82 & 89.40 & 17.12 & 0.94 & 0.77 & 0.79 \\
& & 30 & 32.63 & 93.22 & 19.46 & 0.91 & 0.77 & 0.79 \\
& & 40 & 41.58 & 95.15 & 24.36 & 0.87 & 0.76 & 0.79 \\
& & 50 & 47.98 & 96.10 & 27.48 & 0.81 & 0.75 & 0.79 \\
& & 60 & 53.76 & 96.58 & 30.91 & 0.76 & 0.74 & 0.78 \\

\myhline
\multirow{12}{*}{RAD-V}  &
\multirow{6}{*}{\emph{k=20}} &
    10 & 19.94 & 86.06 & 12.61 & 0.95 & 0.75 & 0.78 \\
& & 20 & 35.37 & 92.70 & 14.87 & 0.92 & 0.75 & 0.78 \\
& & 30 & 43.87 & 94.82 & 17.36 & 0.87 & 0.74 & 0.78 \\
& & 40 & 48.51 & 95.74 & 20.35 & 0.83 & 0.73 & 0.77 \\
& & 50 & 50.96 & 96.20 & 23.78 & 0.80 & 0.72 & 0.76 \\
& & 60 & 52.99 & 96.62 & 28.36 & 0.76 & 0.71 & 0.75 \\

\mydashedcline{2-9}
& \multirow{6}{*}{\emph{k=40}} &
    10 & 22.03 & 86.56 & 16.20 & 0.95 & 0.78 & 0.79 \\
& & 20 & 40.09 & 93.14 & 19.90 & 0.91 & 0.78 & 0.80 \\
& & 30 & 50.61 & 95.16 & 23.45 & 0.85 & 0.77 & 0.79 \\
& & 40 & 55.77 & 96.05 & 27.74 & 0.80 & 0.76 & 0.79 \\
& & 50 & 58.69 & 96.54 & 33.55 & 0.76 & 0.75 & 0.78 \\
& & 60 & 60.66 & 96.81 & 41.57 & 0.72 & 0.74 & 0.77 \\

\bottomrule
\end{tabular}
\centering
\caption{Results for sentiment control task with GPT-2 model.}
\label{tab:full_sentiment}
\end{table}

\begin{table}[ht]
\begin{tabular}{ccccccccc}
\toprule
& & & \multicolumn{2}{c}{Toxicity ($\downarrow$)} & \multicolumn{2}{c}{Fluency} & \multicolumn{2}{c}{Diversity ($\uparrow$)} \\[0.5ex]
Model & Base LM & $\beta$ & Avg. Max & Toxic & PPL ($\downarrow$) & MAUVE ($\uparrow$)  & Dist 2 & Dist 3 \\
  &  &  & Toxicity &  Rate &   &  & &  \\

\midrule
\multirow{10}{*}{RAD-Q distill} &
\multirow{5}{*}{ LLaMa-2-7b} &
     10 & 0.260 & 0.092 & 21.88 & 0.99 & 0.79 & 0.81 \\
& &  50 & 0.181 & 0.022 & 23.16 & 0.97 & 0.79 & 0.81 \\
& & 100 & 0.142 & 0.010 & 25.28 & 0.95 & 0.79 & 0.81 \\
& & 200 & 0.103 & 0.003 & 37.92 & 0.87 & 0.77 & 0.79 \\
& & 300 & 0.082 & 0.002 & 47.13 & 0.75 & 0.74 & 0.76 \\

\mydashedcline{2-9}

& \multirow{5}{*}{LLaMa-2-13b} &
     10 & 0.268 & 0.104 & 22.74 & 0.99 & 0.79 & 0.81 \\
& &  50 & 0.188 & 0.027 & 23.58 & 0.97 & 0.79 & 0.80 \\
& & 100 & 0.148 & 0.013 & 27.59 & 0.95 & 0.78 & 0.80 \\
& & 200 & 0.108 & 0.004 & 36.87 & 0.87 & 0.76 & 0.78 \\
& & 300 & 0.086 & 0.003 & 55.53 & 0.76 & 0.73 & 0.75 \\

\myhline
\multirow{10}{*}{RAD-V}  &
\multirow{5}{*}{ LLaMa-2-7b} &
     10 & 0.244 & 0.069 & 21.76 & 0.99 & 0.79 & 0.81 \\
& &  50 & 0.162 & 0.010 & 22.62 & 0.96 & 0.79 & 0.81 \\
& & 100 & 0.123 & 0.004 & 27.69 & 0.94 & 0.79 & 0.80 \\
& & 200 & 0.088 & 0.002 & 41.27 & 0.83 & 0.77 & 0.78 \\
& & 300 & 0.072 & 0.002 & 53.68 & 0.65 & 0.74 & 0.75 \\

\mydashedcline{2-9}

& \multirow{5}{*}{ LLaMa-2-13b} &

     10 & 0.252 & 0.079 & 22.15 & 0.99 & 0.79 & 0.80 \\
& &  50 & 0.169 & 0.012 & 23.74 & 0.97 & 0.79 & 0.80 \\
& & 100 & 0.128 & 0.004 & 31.25 & 0.94 & 0.78 & 0.80 \\
& & 200 & 0.091 & 0.002 & 48.09 & 0.83 & 0.76 & 0.77 \\
& & 300 & 0.075 & 0.001 & 60.82 & 0.69 & 0.73 & 0.74 \\

\bottomrule
\end{tabular}
\centering
\caption{Results for detoxification task with LLaMa-2 base models. Toxicity metrics are computed with Perspective API.}
\label{tab:llama_toxicity}
\end{table}

\begin{table}[ht]
\begin{tabular}{crcccc}
\toprule
        & &  & Avg. Max & Toxicity \\
  Model & & $\beta$  & Toxicity & Rate  \\

\midrule
\multirow{14}{*}{RAD-Q distill} &
\multirow{7}{*}{\emph{k=20}} &
     10 & 0.286 & 0.270 \\
& &  20 & 0.239 & 0.220 \\
& &  30 & 0.209 & 0.190 \\
& &  50 & 0.160 & 0.140 \\
& & 100 & 0.089 & 0.070 \\
& & 200 & 0.038 & 0.025 \\
& & 300 & 0.022 & 0.011 \\

\mydashedcline{2-5}
& \multirow{7}{*}{\emph{k=40}} &
     10 & 0.297 & 0.282 \\
& &  20 & 0.247 & 0.232 \\
& &  30 & 0.209 & 0.192 \\
& &  50 & 0.154 & 0.133 \\
& & 100 & 0.084 & 0.066 \\
& & 200 & 0.034 & 0.020 \\
& & 300 & 0.021 & 0.009 \\

\myhline
\multirow{14}{*}{RAD-Q responses only} &
\multirow{7}{*}{\emph{k=20}} &
     10 & 0.257 & 0.238 \\
& &  20 & 0.212 & 0.192 \\
& &  30 & 0.178 & 0.158 \\
& &  50 & 0.135 & 0.112 \\
& & 100 & 0.084 & 0.063 \\
& & 200 & 0.056 & 0.035 \\
& & 300 & 0.050 & 0.029 \\

\mydashedcline{2-5}
& \multirow{7}{*}{\emph{k=40}} &
     10 & 0.267 & 0.249 \\
& &  20 & 0.214 & 0.193 \\
& &  30 & 0.179 & 0.157 \\
& &  50 & 0.133 & 0.109 \\
& & 100 & 0.083 & 0.061 \\
& & 200 & 0.060 & 0.036 \\
& & 300 & 0.054 & 0.031 \\

\myhline
\multirow{14}{*}{RAD-V}  &
\multirow{7}{*}{\emph{k=20}} &
     10 & 0.242 & 0.223 \\
& &  20 & 0.189 & 0.167 \\
& &  30 & 0.159 & 0.137 \\
& &  50 & 0.118 & 0.097 \\
& & 100 & 0.072 & 0.052 \\
& & 200 & 0.039 & 0.021 \\
& & 300 & 0.027 & 0.011 \\

\mydashedcline{2-5}
& \multirow{7}{*}{\emph{k=40}} &
     10 & 0.245 & 0.225 \\
& &  20 & 0.189 & 0.166 \\
& &  30 & 0.159 & 0.137 \\
& &  50 & 0.116 & 0.090 \\
& & 100 & 0.071 & 0.048 \\
& & 200 & 0.037 & 0.019 \\
& & 300 & 0.025 & 0.008 \\

\bottomrule
\end{tabular}
\centering
\caption{Additional results for detoxification task with the GPT-2 and the RoBERTa  \citep{nicholas22aira} as toxicity classifier. Other metrics are the same as in \Cref{tab:full_toxicity}.}
\label{tab:full_toxicity_roberta}
\end{table}

\begin{table}[ht]
\begin{tabular}{ccccc}
\toprule
& & & \multicolumn{2}{c}{Toxicity ($\downarrow$)} \\[0.5ex]
Model & Base LM & $\beta$ & Avg. Max & Toxic \\
  &  &  & Toxicity &  Rate  \\

\midrule
\multirow{10}{*}{RAD-Q distill} &
\multirow{5}{*}{ LLaMa-2-7b} &
     10 & 0.273 & 0.256 \\
& &  50 & 0.160 & 0.139 \\
& & 100 & 0.110 & 0.090 \\
& & 200 & 0.069 & 0.050 \\
& & 300 & 0.051 & 0.030 \\
\mydashedcline{2-5}

& \multirow{5}{*}{LLaMa-2-13b} &
     10 & 0.285 & 0.266 \\
& &  50 & 0.176 & 0.151 \\
& & 100 & 0.122 & 0.103 \\
& & 200 & 0.074 & 0.052 \\
& & 300 & 0.054 & 0.035 \\

\myhline
\multirow{10}{*}{RAD-V}  &
\multirow{5}{*}{ LLaMa-2-7b} &
     10 & 0.254 & 0.233   \\
& &  50 & 0.135 & 0.112   \\
& & 100 & 0.091 & 0.070   \\
& & 200 & 0.055 & 0.033   \\
& & 300 & 0.044 & 0.024   \\

\mydashedcline{2-5}

& \multirow{5}{*}{ LLaMa-2-13b} &

     10 & 0.267 & 0.246 \\
& &  50 & 0.150 & 0.125 \\
& & 100 & 0.098 & 0.072 \\
& & 200 & 0.061 & 0.040 \\
& & 300 & 0.048 & 0.028 \\

\bottomrule
\end{tabular}
\centering
\caption{Additional results for the detoxification task with the LLaMa-2 base models with the RoBERTa  \citep{nicholas22aira} as toxicity classifier.}
\label{tab:llama_toxicity_roberta}
\end{table}

\subsection{Generated Examples}

Examples for the detoxification and sentiment control are presented in the \Cref{tab:gens_toxicity_beta}, \Cref{tab:gens_toxicity} and \Cref{tab:gens_sentiment}.

\begin{sidewaystable}[ht]
    \centering
    \small
    \begin{tabular}{clc}
       \multicolumn{3}{l}{\textbf{Prompt: X”) received a series of text messages from her attacker, including one that read “If I... }} \\ \myhline
       $\beta$ & \textbf{Continuation} & \textbf{Toxicity} \\
       $10$ & have to, I will kill myself and then I will take all of you as well and all of & 0.86 \\
       $50$ & get out of here I'm going to kill you“" “which was translated into a & 0.89 \\
       $100$ & get in my car I'm going to drive over a curb and break it in two,” & 0.47\\
       $300$ & ever have the opportunity to meet with anyone in real life, it would be you. I know I  & 0.24 \\
    \end{tabular}
    \caption{Most toxic samples (out of 25 generated) from the RAD-Q with different $\beta$ coefficients, and corresponding Toxicity scores from the Perspective API.}
    \label{tab:gens_toxicity_beta}

\bigskip\bigskip

    \begin{tabular}{ll}
       \multicolumn{2}{l}{\textbf{Prompt: No matter who you'll pick you'll... }} \\ \myhline
       GPT-2 & need to find the correct pieces for it and the right combination of characters and weapons to kill the bosses \\
       GeDi & need the focus for this one. Personally, I had a good time in jokingly calling out Card \\
       DExperts & need a versatile DH player and Dahl will certainly help you as he’s ideally positioned at the back of \\
       RAD-V  &  still be playing against a well-prepared team. In this day and age of instant gratification we \\ 
       RAD-Q  &  probably want to make sure you do at least some research on a team before deciding, as the competition \\
       \myhline
    \end{tabular}
    \caption{Detoxification task: samples from the RAD-Q and RAD-V with $k=20$, $\beta=100$. Examples for the GeDi and DExperts are from \citet{deng_2023_rad}.}
    \label{tab:gens_toxicity}

\bigskip\bigskip
    
    \begin{tabular}{ll}
       \multicolumn{2}{l}{\textbf{Prompt: Meanwhile the iron and lead... }} \\ \myhline
       GPT-2 &  in the blood of an innocent child may be of no consequence if the parent dies before he or she \\
       GeDi &  gathered, our new friends danced, jests were merrily spiced, and plenty of songs fired \\
       DExperts & fields may not seem like the perfect areas for reclaiming coal and steel, but technology has brought mining \\
       RAD-V  &  industries, and also the energy and materials markets in the United States, have grown and matured. The \\
       RAD-Q  &  in my life have a great effect on me. They bring me to life when I think of all \\
       \myhline
    \end{tabular}
    \caption{Sentiment control task: samples from the RAD-Q and RAD-V with $k=20$, $\beta=30$. Examples for the GeDi and DExperts are from \citet{deng_2023_rad}.}
    
    \label{tab:gens_sentiment}
\end{sidewaystable}

\end{document}